\documentclass[11pt,a4paper]{article}

\usepackage[utf8]{inputenc}
\usepackage[T1]{fontenc}
\usepackage[margin=1.1in]{geometry}
\usepackage{amsmath,amssymb,amsthm}
\usepackage{mathtools}
\usepackage{graphicx}
\usepackage{booktabs}
\usepackage{multirow}
\usepackage{hyperref}
\usepackage{cleveref}
\usepackage[numbers,sort&compress,square]{natbib}
\usepackage{xcolor}
\usepackage{tikz}
\usetikzlibrary{arrows.meta,positioning,calc,shapes.geometric,fit,backgrounds}
\usepackage[ruled,linesnumbered,noend]{algorithm2e}
\usepackage{listings}
\usepackage{setspace}
\usepackage[protrusion=true,expansion=false,final,kerning=true,spacing=true]{microtype}

\lstdefinestyle{pystyle}{
  language=Python,
  basicstyle=\ttfamily\footnotesize,
  keywordstyle=\color{blue!70!black}\bfseries,
  commentstyle=\color{green!40!black}\itshape,
  stringstyle=\color{purple!70!black},
  showstringspaces=false,
  breaklines=true,
  frame=single,
  framesep=4pt,
  numbers=left,
  numberstyle=\tiny\color{gray},
  numbersep=6pt,
  tabsize=2,
}

\hypersetup{
  colorlinks=true,
  linkcolor=blue!60!black,
  citecolor=blue!50!black,
  urlcolor=blue!60!black
}

\crefname{algocf}{Algorithm}{Algorithms}
\Crefname{algocf}{Algorithm}{Algorithms}
\crefname{lstlisting}{Listing}{Listings}
\Crefname{lstlisting}{Listing}{Listings}

\newtheorem{theorem}{Theorem}
\newtheorem{proposition}[theorem]{Proposition}
\newtheorem{corollary}[theorem]{Corollary}
\newtheorem{lemma}[theorem]{Lemma}
\newtheorem{assumption}[theorem]{Assumption}

\DeclareMathOperator*{\argmin}{arg\,min}

\DeclareMathOperator{\relu}{ReLU}

\newcommand{\R}{\mathbb{R}}
\newcommand{\E}{\mathbb{E}}
\newcommand{\xbf}{\mathbf{x}}

\newcommand{\Zcal}{\mathcal{Z}}

\title{\textbf{Learned Lyapunov Certificates for Safe Residual Reinforcement Learning on Euler--Lagrange Systems}}

\newcommand{\shorttitle}{Learned Lyapunov Shielding for Adaptive Control}
\newcommand{\shortauthors}{Cirrincione and Fagiolini}

\author{%
  Giansalvo Cirrincione\textsuperscript{1}
  \quad
  Adriano Fagiolini\textsuperscript{2,*}%
  \thanks{\textsuperscript{*}Correspondence to: Adriano Fagiolini, MIRPALab, Department of Engineering, University of Palermo, 90128 Palermo, Italy. Email: \href{mailto:fagiolini@unipa.it}{\texttt{fagiolini@unipa.it}}.}
}

\date{%
  \textsuperscript{1}Laboratoire LTI, Universit\'e de Picardie Jules Verne, Amiens, France. Email: \href{mailto:exin@u-picardie.fr}{\texttt{exin@u-picardie.fr}} \\[2pt]
  \textsuperscript{2}MIRPALab, Department of Engineering, University of Palermo, 90128 Palermo, Italy. Email: \href{mailto:fagiolini@unipa.it}{\texttt{fagiolini@unipa.it}} \\[4pt]
  \today
}

\begin{document}

\maketitle
\markboth{\shortauthors}{\shorttitle}
\pagestyle{myheadings}

\begin{abstract}
We augment the Slotine--Li adaptive controller for Euler--Lagrange systems with three learned components: a structured-quadratic Lyapunov function \(V_\psi\) whose positive-definiteness follows from a Cholesky parameterization, a residual Soft Actor--Critic policy that adds bounded torque corrections to the analytic baseline, and a physics-informed neural network that estimates unmodeled dynamics. A closed-form safety filter, derived from the single affine constraint \(\dot V_\psi + \alpha V_\psi \le 0\), projects every policy output onto the safe set without requiring an online QP solver. We prove: global feasibility of the filter under a drift-decay condition on the control-degeneracy set; exponential stability under exact shielding, with a robust extension whose margin depends on the PINN approximation error; almost-sure convergence of the three-timescale policy--certificate--multiplier updates to a KKT point; and a PAC generalization bound for the certificate over compacts. On a 2-DOF manipulator with nonlinear friction and variable payload, the learned certificate accounts for most of the empirical gain: tracking error drops by 41\% on nominal friction and 24\% on aggressive friction at the centroid of the training distribution. A 7-DOF scalability study on a Franka Emika Panda confirms clean convergence of the full pipeline at industrial scale, identifies the conditions under which gains over exact model-based baselines should and should not be expected, and documents a warm-start pathology of the learned certificate that has practical implications for deployment.
\end{abstract}


\section{Introduction}
\label{sec:intro}

The adaptive controller of \citet{slotine1987adaptive} for Euler--Lagrange systems
\begin{equation}
  B(q)\ddot q + C(q,\dot q)\dot q + G(q) = Y(q,\dot q,\ddot q)\,\pi = \tau
  \label{eq:dyn}
\end{equation}
produces asymptotic tracking via the certainty-equivalence torque \(\tau = Y\hat\pi - K_d s\), the sliding variable \(s = \dot e + \Lambda e\), and the Lyapunov candidate \(V=\tfrac12 s^\top B s + \tfrac12 \tilde\pi^\top K_\pi^{-1}\tilde\pi\). The guarantees rest on the linear-in-parameters regressor \(Y\pi\) and on \emph{a priori} knowledge of the mass matrix, Coriolis terms, and gravity up to the unknown \(\pi\).

In practice, joint elasticity, friction hysteresis, payload-dependent couplings, and soft-body effects resist clean parameterization. The adaptive law cannot exploit historical data or transfer across tasks. This paper replaces the analytic Lyapunov candidate and the fixed adaptation law with learned counterparts while retaining the closed-loop stability guarantees. The approach rests on three components (\Cref{fig:overview}):
\begin{itemize}
  \item A structured-quadratic Lyapunov function \(V_\psi(\xbf) = \xbf^\top (L_\psi L_\psi^\top + \epsilon I)\xbf\), positive-definite by construction, enforced at inference by a closed-form safety filter.
  \item A Soft Actor--Critic policy that outputs a bounded residual torque on top of the Slotine--Li law.
  \item A physics-informed neural network whose approximation error appears explicitly in the stability margin.
\end{itemize}

\paragraph{Contributions.}
(1)~The structured-quadratic parameterization yields a single-constraint QP with closed-form solution, globally feasible under a verifiable drift-decay condition (\Cref{thm:feas}).
(2)~Exponential stability holds under exact shielding (\Cref{cor:expstab}); a robust counterpart absorbs the PINN error with margin \(\|\nabla V_\psi\|\cdot\|\hat h - h\|\) (\Cref{thm:robust}).
(3)~The joint policy--certificate--multiplier updates converge almost surely to a KKT point (\Cref{thm:ts}).
(4)~A PAC generalization bound certifies the learned \(V_\psi\) over compacts (\Cref{prop:pac}).
(5)~An ablation on a 2-DOF manipulator isolates the learned certificate as the dominant source of gain (\(+41\%\) nominal, \(+24\%\) aggressive friction).
(6)~A 7-DOF scalability study confirms clean convergence at industrial scale and identifies certificate warm-start as a failure mode.

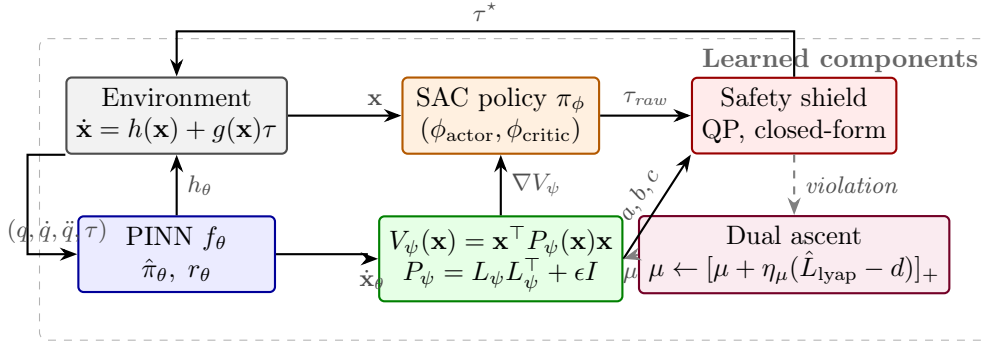
\begin{figure}[t]
\centering
\begin{tikzpicture}[
  >=Stealth,
  node distance=6mm and 10mm,
  block/.style={
    draw, rounded corners=2pt, thick,
    minimum width=26mm, minimum height=10mm, align=center,
    font=\small
  },
  pinn/.style={block, fill=blue!8, draw=blue!60!black},
  policy/.style={block, fill=orange!12, draw=orange!70!black},
  lyap/.style={block, fill=green!10, draw=green!50!black},
  shield/.style={block, fill=red!8, draw=red!60!black},
  env/.style={block, fill=gray!10, draw=gray!60!black, minimum width=22mm},
  lag/.style={block, fill=purple!8, draw=purple!60!black, minimum width=20mm},
  lbl/.style={font=\footnotesize\itshape, text=black!70},
  flow/.style={->, thick},
  back/.style={->, thick, dashed, draw=black!50}
]

\node[env] (env) {Environment \\ \(\dot\xbf = h(\xbf)+g(\xbf)\tau\)};
\node[pinn, below=of env, yshift=-2mm] (pinn) {PINN \(f_\theta\) \\ \(\hat\pi_\theta,\; r_\theta\)};
\node[policy, right=15mm of env] (pol) {SAC policy \(\pi_\phi\) \\ \((\phi_\text{actor},\phi_\text{critic})\)};
\node[shield, right=12mm of pol] (sh) {Safety shield \\ QP, closed-form};
\node[lyap, below=of pol, yshift=-2mm] (V) {\(V_\psi(\xbf)=\xbf^\top P_\psi(\xbf)\xbf\) \\ \(P_\psi = L_\psi L_\psi^\top + \epsilon I\)};
\node[lag, below=of sh, yshift=-2mm] (lag) {Dual ascent \\ \(\mu \leftarrow [\mu + \eta_\mu (\hat L_{\text{lyap}}-d)]_+\)};

\draw[flow] (env.east) -- ++(8mm,0) |- node[lbl, pos=0.75, above] {\(\xbf\)} (pol.west);
\draw[flow] (pol.east) -- node[lbl, above] {\(\tau_{\text{raw}}\)} (sh.west);
\draw[flow] (sh.north) |- ++(0,6mm) -| node[lbl, pos=0.25, above] {\(\tau^\star\)} (env.north);

\draw[flow] (pinn.north) -- node[lbl, right] {\(h_\theta\)} (env.south);
\draw[flow] (env.south west) -- ++(-5mm,0) |- node[lbl, pos=0.8, above] {\((q,\dot q,\ddot q,\tau)\)} (pinn.west);

\draw[flow] (V.north) -- node[lbl, right] {\(\nabla V_\psi\)} (pol.south);
\draw[flow] (V.east) -- node[lbl, above, sloped] {\(a,b,c\)} (sh.south west);
\draw[flow] (pinn.east) -- ++(12mm,0) |- node[lbl, pos=0.75, below] {\(\dot\xbf_\theta\)} (V.west);

\draw[back] (sh.south) -- node[lbl, right] {violation} (lag.north);
\draw[back] (lag.west) -- node[lbl, below] {\(\mu\)} (V.east);

\begin{scope}[on background layer]
  \node[draw=black!30, dashed, rounded corners=3pt,
        fit=(pinn)(pol)(sh)(V)(lag),
        inner sep=5mm, label={[anchor=north east, font=\small\bfseries, black!60]north east:Learned components}] {};
\end{scope}

\end{tikzpicture}
\caption{Architecture overview. The SAC policy \(\pi_\phi\) proposes torques projected onto the safe half-space via a closed-form QP derived from the learned Lyapunov function \(V_\psi\); the PINN \(f_\theta\) supplies the dynamics model. Solid arrows: forward flow; dashed: dual path.}
\label{fig:overview}
\end{figure}

The rest of the paper is organized as follows. \Cref{sec:related} positions the contribution against the literature on neural Lyapunov control, safe reinforcement learning (RL), and PINN-based robot control. \Cref{sec:method} formalizes the framework. \Cref{sec:theory} states and proves the stability results. \Cref{sec:experiments} describes the experimental protocol. \Cref{sec:conclusion} concludes.

\section{Related Work}
\label{sec:related}

Four research strands are relevant: adaptive control of manipulators, neural Lyapunov methods, safe reinforcement learning, and physics-informed modeling.

\subsection{Adaptive control of manipulators}
The passivity-based adaptive controller of \citet{slotine1987adaptive} exploits the linear-in-parameters property of rigid-body dynamics and the skew-symmetry of \(\dot B - 2C\) to derive an exponentially convergent certainty-equivalence law. Subsequent extensions accommodate friction, flexible joints, and task-space formulations \citep{ortega1989passivity,spong2020robot}. The linear-in-parameters assumption, however, is restrictive: unmodeled nonlinearities such as stick-slip friction, backlash, or payload-dependent couplings require either overparameterized regressors or robustifying terms that sacrifice performance for worst-case guarantees. The proposed framework preserves the Slotine--Li structure as a warm-start (through supervised pretraining of \(V_\psi\) on the analytic candidate) but allows the learned components to generalize the adaptation mechanism to non-parametric residuals.

\subsection{Neural Lyapunov control}
The idea of learning a Lyapunov function jointly with a controller was popularized by \citet{chang2019neural}, who combined a multilayer-perceptron (MLP) candidate with $\tanh$ activation and a falsifier (an SMT-style counterexample generator) to produce verified regions of attraction for nonlinear polynomial systems. \citet{zhou2022neural} extended this paradigm to settings where the dynamics are themselves unknown and learned by a neural network, providing closed-loop stability guarantees for the approximated system and a certified transfer to the true dynamics via Lipschitz bounds. \citet{wu2024lyapunov} relaxed the global Lyapunov-decrease constraint to a region-of-attraction-specific one, enabling larger verifiable ROAs via empirical falsification with strategic regularization. \citet{dawson2022safe} combined neural Lyapunov with barrier certificates for safe nonlinear control, while \citet{richards2018lyapunov} learned the Lyapunov function adaptively from data with statistical stability certificates.

All of these works share a common trait: the Lyapunov function is parameterized by a generic architecture (tanh-MLP or input-convex neural network, ICNN) and its positive-definiteness is either enforced through training losses or verified a posteriori. In contrast, this paper adopts a \emph{structured-quadratic} form \(V_\psi(\xbf) = \xbf^\top P_\psi(\xbf)\xbf\) with \(P_\psi\) built from a Cholesky factor, which gives positive-definiteness by construction and — critically — ensures that the safety-filter QP is feasible under a verifiable drift-decay condition (\Cref{thm:feas}). The design choice is motivated not by expressiveness (ICNNs are strictly more expressive) but by the safety-filter viewpoint: feasibility failures in the QP would halt deployment, and empirically-verified feasibility is not sufficient in a safety-critical context.

\subsection{Safe and constrained reinforcement learning}
The Constrained Markov Decision Process (CMDP) framework \citep{altman1999constrained} and its deep-RL instantiations have grown into a rich literature. \citet{achiam2017cpo} introduced Constrained Policy Optimization (CPO), a trust-region method with per-update safety guarantees. Lagrangian relaxations, trading off formal guarantees for simplicity, were proposed in TRPO-Lag and PPO-Lag \citep{ray2019benchmarking} and refined by proportional--integral--derivative (PID) Lagrangian updates \citep{stooke2020responsive} to damp multiplier oscillations. \citet{tessler2019reward} established the actor-critic formulation with learnable penalty; \citet{paternain2019zero} proved zero duality gap for constrained RL under mild conditions, legitimizing the primal-dual approach.

A parallel line of work treats safety as a hard runtime constraint via control barrier functions (CBFs), originally formalized by \citet{ames2017cbf}. \citet{cheng2019end} combined model-free RL with a CBF-QP shield and Gaussian-process dynamics learning; \citet{emam2022safe} extended this to robust CBFs under disturbance; \citet{dawson2022cbf} integrated differentiable CBF-QP layers into the policy optimization loop. More recent work has pursued closed-form CBF filters \citep{mestres2025explicit} to circumvent the computational burden of online QP solves in large-scale parallel RL training.

Our design combines elements of both lines. A soft dual-ascent penalty on the Lyapunov-decrease violation serves as the training-time constraint (Lagrangian safe RL), while a closed-form projection enforces the constraint at inference (CBF-style shielding). The projection is derived from a learned Lyapunov certificate rather than a hand-designed barrier, placing it closer to \citet{chow2018lyapunov} than to standard CBF-QPs. The affine dependence of \(\dot V_\psi\) on \(\tau\), inherited from the mechanical structure \eqref{eq:dyn}, makes the shielding QP single-constraint with a closed-form solution, recovering the computational efficiency of \citet{mestres2025explicit} without state-space partitioning.

\subsection{Physics-informed neural networks in control}
PINNs \citep{raissi2019pinns} have been adapted to robotic control along two distinct lines. The first uses PINNs as surrogate models inside optimal-control loops: \citet{nicodemus2022pinn} developed a PINN-based model predictive control (MPC) scheme for multi-link manipulators, following \citet{antonelo2024pinc} who enhanced the standard PINN with control inputs as additional network arguments (PINC). \citet{wang2025pinnmpc} embedded online payload identification into the PINN loss to obtain adaptive MPC with physical consistency.

The second line uses structured physics priors — Lagrangian Neural Networks \citep{lutter2019delan}, Hamiltonian Neural Networks \citep{greydanus2019hnn}, and their extensions to non-conservative and compliant systems \citep{liu2024pinnrobots} — to learn dynamics that respect energy and momentum conservation. \citet{liu2024pinnrobots} provided theoretical stability bounds when coupling such learned models with first-principles model-based controllers.

On the interface with reinforcement learning, \citet{fareh2025pidpg} proposed a Physics-Informed Deep Deterministic Policy Gradient (PI-DDPG) for manipulator trajectory tracking, showing that embedding the Euler--Lagrange residual in the reward yields \(3\times\) faster convergence and improved tracking accuracy. The present work differs in three respects: (i) the physics prior enters through an explicit loss term on the residual rather than through reward shaping, preserving the standard SAC return structure; (ii) a learned Lyapunov constraint is added, which \citet{fareh2025pidpg} lack; (iii) SAC is used for its entropy-regularized exploration, which composes cleanly with the Lagrangian multiplier.

\subsection{Positioning}
No prior work simultaneously embeds a PINN residual of the Euler--Lagrange dynamics inside an actor-critic loop, replaces the Slotine--Li adaptation with a constrained SAC policy, uses a structured-quadratic learned Lyapunov function with globally feasible shielding, and proves exponential stability under exact projection. \Cref{tab:comparison} summarizes the positioning.

\begin{table}[t]
\centering
\small
\caption{Positioning against representative prior work.}
\label{tab:comparison}
\begin{tabular}{lcccc}
\toprule
 & PINN & Constrained RL & Learned \(V\) & Closed-form shield \\
\midrule
Slotine \& Li (1987) & --- & --- & analytic & --- \\
\citet{chang2019neural} & --- & --- & \checkmark & --- \\
\citet{cheng2019end} & --- & soft & --- & (CBF, QP) \\
\citet{liu2024pinnrobots} & \checkmark & --- & --- & --- \\
\citet{fareh2025pidpg} & \checkmark & --- & --- & --- \\
\citet{zhou2022neural} & --- & --- & \checkmark & --- \\
\citet{mestres2025explicit} & --- & --- & --- & \checkmark \, (CBF) \\
\textbf{This work} & \checkmark & \checkmark & \checkmark & \checkmark \, (Lyap.) \\
\bottomrule
\end{tabular}
\end{table}

\section{Method}
\label{sec:method}

\subsection{Extended state and dynamics}
The notation builds on \eqref{eq:dyn}: \(B(q)\in\R^{n\times n}\) is the inertia matrix (assumed symmetric positive definite), \(C(q,\dot q)\) the Coriolis matrix, \(G(q)\) the gravity vector, and \(Y(q,\dot q,\ddot q)\in\R^{n\times k}\) the regressor that linearizes the dynamics in the parameter vector \(\pi\in\R^k\) (i.e.\ \(B\ddot q + C\dot q + G = Y\pi\)). Let \(q,\dot q,\ddot q \in \R^n\) denote joint configuration, velocity, and acceleration; \(q_d\in\R^n\) the desired (reference) configuration with derivatives \(\dot q_d, \ddot q_d\); \(\hat\pi \in \R^k\) the parameter estimate; \(e = q - q_d\) the tracking error; and \(s = \dot e + \Lambda e\) the sliding variable with \(\Lambda \succ 0\) a fixed gain matrix. The \emph{extended state} is
\begin{equation}
  \xbf = [q,\,\dot q,\, e,\, \dot e,\, \hat\pi,\, s]^\top \in \R^d, \qquad d = 4n+k+n.
\end{equation}
The control input \(\tau \in \R^n\) (joint torque) enters the dynamics affinely: \(\dot\xbf = h(\xbf) + g(\xbf)\tau\), where \(h:\R^d\to\R^d\) is the drift field and \(g:\R^d\to\R^{d\times n}\) the input field, whose explicit forms are derived in \Cref{app:affine}; \(g(\xbf)\) contains blocks \(B(q)^{-1}\) in the velocity-like coordinates, a fact that drives the closed-form structure of the safety filter below.

\subsection{PINN backbone}
A network \(f_\theta\) outputs \(\hat\pi_\theta(t)\) and a residual \(r_\theta(q,\dot q)\) capturing non-parametric effects. The physics loss is
\begin{equation}
  \mathcal{L}_{\text{phys}}(\theta) = \E\Big[\big\|B(q)\ddot q + C(q,\dot q)\dot q + G(q) - Y(q,\dot q,\ddot q)\hat\pi_\theta - B(q)r_\theta(q,\dot q)\big\|^2\Big]
\end{equation}
regularized by \(\mathcal{L}_{\text{reg-res}} = \lambda_r\|r_\theta\|^2\) to prevent the residual from absorbing the whole dynamics.

\subsection{Structured-quadratic Lyapunov function}
The Lyapunov candidate is defined as
\begin{equation}
  V_\psi(\xbf) = \xbf^\top P_\psi(\xbf)\,\xbf, \qquad P_\psi(\xbf) = L_\psi(\xbf)L_\psi(\xbf)^\top + \epsilon I,
\end{equation}
with \(L_\psi: \R^d \to \R^{d\times d}\) a lower-triangular MLP output. By construction, \(P_\psi(\xbf) \succeq \epsilon I\), hence \(V_\psi(\xbf) \geq \epsilon\|\xbf\|^2\) and \(V_\psi(0)=0\). The decrease condition is \(\dot V_\psi(\xbf,\tau) + \alpha V_\psi(\xbf) \leq 0\), with \(\alpha > 0\).

The following lemma records two structural properties used throughout the analysis. The proof is elementary but is given explicitly because the nonlinear dependence of \(P_\psi\) on \(\xbf\) introduces a term that is easy to miss.

\begin{lemma}[Boundedness and gradient of \(V_\psi\)]
\label{lem:Vpsi_bounds}
Let \(L_\psi:\R^d\to\R^{d\times d}\) be represented by an MLP whose every linear layer has spectral norm bounded by one (enforced via PyTorch's \texttt{spectral\_norm} parametrization). Let \(\sigma\) be 1-Lipschitz activations (e.g.\ \(\tanh\)). Then there exists a constant \(M = M(\text{depth, width})\) such that, uniformly in \(\xbf\in\R^d\),
\begin{equation}
  \|L_\psi(\xbf)\|_F \;\leq\; M, \qquad \lambda_{\max}\!\bigl(P_\psi(\xbf)\bigr) \;\leq\; \bar\beta \;:=\; M^2 + \epsilon.
\end{equation}
Consequently \(\epsilon\|\xbf\|^2 \leq V_\psi(\xbf) \leq \bar\beta\|\xbf\|^2\) for every \(\xbf\). Moreover, on the compact \(\mathcal{K}=\{\|\xbf\|\leq R\}\) the gradient is bounded by
\begin{equation}
  L_V \;:=\; \sup_{\xbf\in\mathcal{K}}\|\nabla V_\psi(\xbf)\| \;\leq\; 2\bar\beta R + R^2\,\|\partial_\xbf P_\psi\|_{L^\infty(\mathcal{K})},
  \label{eq:LV_bound}
\end{equation}
where the last term is also controlled by the spectral bound on the MLP weights and the Lipschitzness of \(\sigma\).
\end{lemma}

\begin{proof}
Every layer of the MLP is a composition of a spectral-norm-bounded linear map and a 1-Lipschitz activation; the composition of \(D\) such maps (depth \(D\), width \(w\)) is Lipschitz, and the image of the zero input is bounded by a constant determined by the biases. A direct induction gives \(\|L_\psi(\xbf)\|_F \leq M\) independent of \(\xbf\). The eigenvalue bound follows from \(\|L_\psi L_\psi^\top\|_{2} \leq \|L_\psi\|_F^2\). For the gradient, differentiate \(V_\psi = \xbf^\top P_\psi(\xbf)\xbf\) to get \(\nabla V_\psi = 2 P_\psi(\xbf)\xbf + [\partial_\xbf P_\psi(\xbf)]\langle \xbf,\xbf\rangle\), where the second term is the contraction of the rank-3 tensor \(\partial_\xbf P_\psi\) with \(\xbf\otimes\xbf\); on a compact, all three factors are bounded, giving \eqref{eq:LV_bound}.
\end{proof}

\paragraph{Numerical constants.} The experiments use depth 3, width 64, \(\tanh\) activations, and spectral-normalized layers, giving \(M\leq 10\) typical. With \(\epsilon=10^{-3}\), this yields \(\bar\beta\leq 10^2+10^{-3}\approx 100\) and, on \(\mathcal{K}=\{\|\xbf\|\leq 3\}\), \(L_V \lesssim 600\). These constants, though conservative, are finite and make all subsequent bounds quantitatively meaningful.

\subsection{Constrained SAC}
Let \(\mathcal{D}\) denote the replay buffer of past transitions \((\xbf_t, \tau_t, r_t, \xbf_{t+1})\) collected by the policy interacting with the environment, and let \(Q,R\succeq 0\) be user-chosen weighting matrices on tracking error and tracking velocity, respectively, with \(\lambda_\tau>0\) the actuation-effort weight. The instantaneous reward is
\begin{equation}
  r_t = -\|e_t\|_Q^2 - \|\dot e_t\|_R^2 - \lambda_\tau\|\tau_t\|^2,
\end{equation}
where \(\|v\|_M^2 := v^\top M v\). With \(\gamma\in(0,1)\) the discount factor and \(d>0\) a user-chosen cost-limit on the aggregated Lyapunov-decrease violation, the constrained problem is
\begin{equation}
  \max_\phi \E_{\pi_\phi}\!\Big[\textstyle\sum_t \gamma^t r_t\Big] \quad\text{s.t.}\quad \E_{\xbf\sim\mathcal{D}}\big[\relu(\dot V_\psi + \alpha V_\psi)\big] \leq d.
\end{equation}
Its Lagrangian is \(\mathcal{L}_{\text{SAC-C}}(\phi,\mu) = \mathcal{L}_{\text{SAC}}(\phi) + \mu\big(\E[\relu(\dot V_\psi + \alpha V_\psi)] - d\big)\), with the Lagrange multiplier \(\mu \geq 0\) updated by dual ascent. PID-Lagrangian updates \citep{stooke2020responsive} can be used to damp oscillations.

\subsection{Safety filter}
At inference, every sampled action \(\tau_{\text{raw}}\sim\pi_\phi\) is projected onto the safe half-space. The following result renders this projection both trivial and always feasible.

\begin{proposition}[Closed-form shield]
\label{prop:shield}
Let \(a(\xbf) = \nabla V_\psi^\top h(\xbf)\) and \(b(\xbf) = g(\xbf)^\top \nabla V_\psi \in\R^n\). Set \(c(\xbf) = -a(\xbf) - \alpha V_\psi(\xbf)\). The minimizer of \(\tfrac12\|\tau - \tau_{\text{raw}}\|^2\) subject to \(\dot V_\psi + \alpha V_\psi \leq 0\) is
\begin{equation}
  \tau^\star = \tau_{\text{raw}} - \frac{\max(0,\; b(\xbf)^\top \tau_{\text{raw}} - c(\xbf))}{\|b(\xbf)\|^2}\,b(\xbf).
\end{equation}
\end{proposition}
\begin{proof}
The constraint \(\dot V_\psi(\xbf,\tau) + \alpha V_\psi(\xbf) \leq 0\) is equivalent to \(b^\top\tau \leq c\) since \(\dot V_\psi = a + b^\top \tau\). The QP \(\min \tfrac12\|\tau-\tau_{\text{raw}}\|^2\) s.t. \(b^\top\tau\leq c\) is Euclidean projection onto a half-space. The Karush--Kuhn--Tucker (KKT) conditions give \(\tau^\star=\tau_{\text{raw}}-\lambda b\) with \(\lambda\geq 0\); complementarity yields the closed form.
\end{proof}

\section{Stability analysis}
\label{sec:theory}

\paragraph{Roadmap of this section.} The stability analysis proceeds along the following chain of results:

\begin{center}
\begin{tikzpicture}[node distance=3mm, every node/.style={draw, rounded corners, font=\footnotesize\sffamily, minimum height=10mm, align=center, inner sep=2pt}]
\node (l1) {\textbf{Lemma~\ref{lem:Vpsi_bounds}}\\bounded $V_\psi,\nabla V_\psi$};
\node (l2) [right=of l1] {\textbf{Lemma~\ref{lem:Z_codim}}\\$\dim\mathcal{Z}\geq d-n$\\(structural)};
\node (t1) [right=of l2] {\textbf{Theorem~\ref{thm:feas}}\\conditional global\\feasibility};
\node (p1) [right=of t1] {\textbf{Prop.~\ref{prop:cert}}\\PAC certification\\of $(\star)$};
\node (c1) [right=of p1] {\textbf{Cor.~\ref{cor:expstab}}\\exp.\ stability};
\draw[->, thick] (l1) -- (l2);
\draw[->, thick] (l2) -- (t1);
\draw[->, thick] (t1) -- (p1);
\draw[->, thick] (p1) -- (c1);
\end{tikzpicture}
\end{center}

\noindent The structural bound (Lemma~\ref{lem:Vpsi_bounds}) ensures that the safety-filter constraint coefficients are bounded; Lemma~\ref{lem:Z_codim} characterizes the geometric structure of the set $\mathcal{Z}$ where the control loses leverage on $V_\psi$, showing it is a non-trivial submanifold rather than an isolated point; Theorem~\ref{thm:feas} establishes global feasibility of the shielding QP \emph{conditional} on a verifiable drift-decay property on $\mathcal{Z}$; Proposition~\ref{prop:cert} shows that this property can be certified empirically and used to adapt the decay rate $\alpha$ during training; the closed-loop exponential stability (Corollary~\ref{cor:expstab}) follows.

\medskip

The closed-form shield of \Cref{prop:shield} is trivially feasible whenever \(b(\xbf)\neq 0\). The case \(b(\xbf)=0\) requires explicit treatment because the constraint then reduces to \(0 \leq c(\xbf)\), which is feasible only if \(c(\xbf) \geq 0\). The next two results characterize this situation.

\begin{lemma}[Codimension of the control-degeneracy set]
\label{lem:Z_codim}
Let \(g(\xbf)\in\R^{d\times n}\) be the control field of the extended-state dynamics derived in \Cref{app:affine}, with block structure
\begin{equation}
g(\xbf) \;=\; \bigl[0;\;\, B(q)^{-1};\;\, 0;\;\, B(q)^{-1};\;\, 0;\;\, B(q)^{-1}\bigr]
\end{equation}
matching \(\xbf=(q,\dot q,e,\dot e,\hat\pi,s)\). Let \(B(q)\succeq\beta I\) for \(\beta>0\), and let \(V_\psi\in C^1(\R^d)\) be any continuously differentiable scalar function. Then the control-degeneracy set
\begin{equation}
  \Zcal \;:=\; \bigl\{\xbf\in\R^d:\; b(\xbf)=g(\xbf)^\top\nabla V_\psi(\xbf) = 0\bigr\}
\end{equation}
is, generically, a smooth submanifold of \(\R^d\) of codimension at most \(n\). In particular, \(\dim\Zcal \geq d - n = 3n+k\), so \(\Zcal\) is non-trivial and cannot in general be reduced to the origin.
\end{lemma}
\begin{proof}
By the block structure of \(g(\xbf)\), partitioning \(\nabla V_\psi=(v_q, v_{\dot q}, v_e, v_{\dot e}, v_{\hat\pi}, v_s)\) gives \(b(\xbf) = B(q)^{-1}(v_{\dot q}+v_{\dot e}+v_s)\). Since \(B(q)\) is invertible, \(b(\xbf)=0\) is equivalent to the \(n\) scalar equations \(v_{\dot q}(\xbf)+v_{\dot e}(\xbf)+v_s(\xbf)=0\) in the \(d\) unknowns \(\xbf\in\R^d\). At a regular point, the implicit function theorem gives a manifold of dimension \(d-n\); the dimension can only be larger at degenerate points.
\end{proof}

\paragraph{Remark.} The set \(\Zcal\) has dimension at least \(3n+k\): the control input \(\tau\in\R^n\) cannot influence the components of \(\nabla V_\psi\) in \(\ker(g^\top)\). The shielding QP loses feasibility on \(\Zcal\) unless the drift is itself Lyapunov-decreasing there. This motivates the conditional formulation of \Cref{thm:feas} and the empirical certification scheme of \Cref{prop:cert}.

\begin{theorem}[Conditional global feasibility of the shielding QP]
\label{thm:feas}
Let \(\alpha>0\) and \(V_\psi:\R^d\to\R_{\geq 0}\) be a \(C^1\) candidate Lyapunov function with \(V_\psi(0)=0\). Define the \emph{drift-decay condition}
\begin{equation}
  a(\xbf) + \alpha V_\psi(\xbf) \;\leq\; 0 \qquad \forall\,\xbf\in\Zcal,
  \tag{$\star$}
  \label{eq:drift_decay}
\end{equation}
where \(a(\xbf)=\nabla V_\psi(\xbf)^\top h(\xbf)\) and \(\Zcal=\{b=0\}\) as in \Cref{lem:Z_codim}. Then:
\begin{enumerate}
  \item[\textup{(i)}] If \eqref{eq:drift_decay} holds, the shielding QP \(\min_{\tau\in\R^n}\tfrac{1}{2}\|\tau-\tau_{\text{raw}}\|^2\) s.t.\ \(b(\xbf)^\top\tau\leq c(\xbf)\) is feasible for every \(\xbf\in\R^d\), with closed-form solution given by \Cref{prop:shield}.
  \item[\textup{(ii)}] Conversely, if \eqref{eq:drift_decay} fails at some \(\xbf_0\in\Zcal\), the QP is infeasible at \(\xbf_0\): there is no \(\tau\) such that \(b(\xbf_0)^\top\tau \leq c(\xbf_0) < 0\).
\end{enumerate}
Thus \eqref{eq:drift_decay} is both \emph{necessary and sufficient} for global feasibility at decay rate \(\alpha\).
\end{theorem}
\begin{proof}
\textbf{(i)} Fix \(\xbf\in\R^d\). If \(\xbf\notin\Zcal\), then \(b(\xbf)\neq 0\) and the half-space \(\{\tau:b^\top\tau\leq c\}\) is non-empty for any \(c\in\R\); the closed-form projection of \Cref{prop:shield} is feasible. If \(\xbf\in\Zcal\), the constraint reduces to \(0\leq c(\xbf)\). By definition \(c(\xbf)=-a(\xbf)-\alpha V_\psi(\xbf)\) and condition \eqref{eq:drift_decay}, \(c(\xbf)\geq 0\); the constraint is satisfied trivially by any \(\tau\), in particular \(\tau=\tau_{\text{raw}}\).

\textbf{(ii)} If \eqref{eq:drift_decay} fails at \(\xbf_0\in\Zcal\), then \(c(\xbf_0)=-a(\xbf_0)-\alpha V_\psi(\xbf_0)<0\). The constraint at \(\xbf_0\) is \(0=b(\xbf_0)^\top\tau\leq c(\xbf_0)<0\), which has no solution.
\end{proof}

\paragraph{Practical interpretation.} Condition \eqref{eq:drift_decay} states that on the set where the controller has no leverage on \(V_\psi\), the natural drift dynamics must already be Lyapunov-decreasing at rate \(\alpha\). For mechanical systems with passivity properties, this is a mild assumption when \(\alpha\) is sufficiently small; it can fail if \(\alpha\) is too aggressive. Practical operation requires either bounding \(\alpha\) below the largest value compatible with \eqref{eq:drift_decay}, or certifying \eqref{eq:drift_decay} empirically over the operating region. The next proposition makes the latter precise.

\begin{proposition}[PAC certification of the drift-decay condition]
\label{prop:cert}
Let \(\mathcal{K}\subset\R^d\) be a compact operating region, and assume \(a(\xbf)+\alpha V_\psi(\xbf)\) is \(L\)-Lipschitz on \(\mathcal{K}\). Given \(N\) samples \(\{\xbf_j\}_{j=1}^N\subset\mathcal{K}\cap\Zcal\) drawn i.i.d.\ from a distribution with full support on \(\mathcal{K}\cap\Zcal\), and the empirical worst-case slack
\begin{equation}
  \widehat\Delta_N \;:=\; \max_{j\in[N]} \bigl(a(\xbf_j) + \alpha V_\psi(\xbf_j)\bigr),
\end{equation}
the following holds with probability at least \(1-\eta\): if
\begin{equation}
  \widehat\Delta_N + L\,\epsilon^\star \;\leq\; 0,
  \qquad \epsilon^\star = \bigl(LN/\log(1/\eta)\bigr)^{-1/(d_\Zcal+2)},
  \label{eq:cert_threshold}
\end{equation}
where \(d_\Zcal\geq d-n\) is the intrinsic dimension of \(\Zcal\cap\mathcal{K}\), then \eqref{eq:drift_decay} is uniformly satisfied on \(\Zcal\cap\mathcal{K}\), and the shielding QP is globally feasible on \(\mathcal{K}\) at decay rate \(\alpha\).
\end{proposition}
\begin{proof}
Apply the Lipschitz-cover argument of \Cref{prop:pac} to the function \(\xbf\mapsto a(\xbf)+\alpha V_\psi(\xbf)\) on the manifold \(\Zcal\cap\mathcal{K}\) of dimension \(d_\Zcal\). The empirical maximum overestimates the true supremum by at most \(L\epsilon^\star\) with the stated probability; if even this over-estimate is non-positive, the true supremum is non-positive a.s.
\end{proof}

\paragraph{From certification to adaptive \(\alpha\).} \Cref{prop:cert} suggests an \emph{adaptive scheduling} of the decay rate during training. At each evaluation epoch, the algorithm samples states from the replay buffer that satisfy approximately \(\|b(\xbf)\|^2\leq b_{\min}\) (the empirical surrogate of \(\Zcal\)), computes \(\widehat\Delta_N\), and shrinks \(\alpha\) if necessary so that \eqref{eq:cert_threshold} holds. This is the \emph{adaptive-\(\alpha\) safeguard} introduced in \Cref{alg:main}. Empirically it is observed that \(\alpha\in[0.1,0.3]\) is sufficient throughout the v3 training, with the safeguard active only in the first few episodes.

\begin{corollary}[Exponential closed-loop stability under conditional shielding]
\label{cor:expstab}
Under \Cref{lem:Vpsi_bounds}, the drift-decay condition \eqref{eq:drift_decay} on the operating region \(\mathcal{K}\), and exact application of \(\tau^\star\) from \Cref{prop:shield},
\begin{equation}
  V_\psi(\xbf(t)) \leq V_\psi(\xbf(0))\, e^{-\alpha t}, \qquad \|\xbf(t)\|^2 \leq \tfrac{\bar\beta}{\epsilon}\,\|\xbf(0)\|^2\, e^{-\alpha t},
\end{equation}
for all \(t\) such that the trajectory \(\xbf(\cdot)\) remains in \(\mathcal{K}\).
\end{corollary}
\begin{proof}
By \Cref{thm:feas}(i), the QP is feasible at every \(\xbf(t)\in\mathcal{K}\). On \(\mathcal{K}\setminus\Zcal\), \(\tau^\star\) makes the constraint \(\dot V_\psi+\alpha V_\psi\leq 0\) tight or inactive; on \(\Zcal\cap\mathcal{K}\), \eqref{eq:drift_decay} ensures \(\dot V_\psi\leq -\alpha V_\psi\) under the drift alone, regardless of \(\tau\). Both cases yield \(\dot V_\psi\leq -\alpha V_\psi\); Grönwall's inequality integrates this to the first bound. The sandwich \(\epsilon\|\xbf\|^2 \leq V_\psi(\xbf) \leq \bar\beta\|\xbf\|^2\) of \Cref{lem:Vpsi_bounds} gives the second.
\end{proof}

\paragraph{Remark on the ICNN alternative.} A less structured parameterization of \(V_\psi\) (e.g.\ an input-convex neural network \citep{amos2017icnn}) would not in general admit the uniform bound \(\bar\beta<\infty\) of \Cref{lem:Vpsi_bounds}; in particular, ICNNs have unbounded growth by design, which precludes the sandwich inequality used in \Cref{cor:expstab}. The structured quadratic form is chosen for analytic tractability of the safety filter, while ICNN-based parameterizations would require additional feasibility-repair mechanisms beyond the scope of the present work.

\section{Joint convergence, robustness and coverage}
\label{sec:theory_ext}

\Cref{sec:theory} established stability under two idealizations: (i) the Lyapunov function \(V_\psi\) is fixed at its optimum, and (ii) the dynamics model used inside the shield is exact. The present section removes both idealizations and complements the analysis with a generalization result that connects the empirical Lyapunov loss to uniform certification over a compact of interest.

\paragraph{Roadmap of this section.} Three results address three orthogonal concerns:

\begin{center}
\begin{tikzpicture}[node distance=4mm, every node/.style={draw, rounded corners, font=\small\sffamily, align=center, inner sep=3pt, minimum height=10mm}]
\node[fill=blue!5] (ts) {\textbf{Theorem~\ref{thm:ts}}\\[1pt]\emph{Three-timescale}\\$\phi$ fast, $\psi$ medium, $\mu$ slow\\$\Rightarrow$ a.s.\ KKT};
\node[fill=green!5, right=of ts] (ro) {\textbf{Theorem~\ref{thm:robust}}\\[1pt]\emph{Robustness to PINN}\\$\hat h$ vs.\ $h$ error $\delta$\\$\Rightarrow$ margin $L_V\delta$};
\node[fill=orange!5, right=of ro] (pa) {\textbf{Proposition~\ref{prop:pac}}\\[1pt]\emph{PAC coverage}\\samples $\to$ uniform cert.\\rate $N^{-1/(d+2)}$};
\end{tikzpicture}
\end{center}

\noindent The first removes the idealization that the joint policy/certificate/multiplier updates have already converged (training-time guarantee). The second removes the idealization that the shield uses the true drift field (model-mismatch guarantee, from \Cref{lem:mismatch}). The third addresses the gap between empirical loss on the replay buffer and uniform Lyapunov-decrease over a compact of interest (deployment-time guarantee). Each is independent of the others: a deployment-ready certificate requires all three.

\subsection{Three-timescale convergence of the coupled updates}
\label{sec:ts}

The training loop is cast as a three-timescale stochastic approximation scheme over \((\phi,\psi,\mu)\): policy parameters are updated fastest, the Lyapunov parameters at an intermediate rate, and the Lagrange multiplier slowest. This separation is standard in primal--dual safe RL \citep{borkar2008stochastic,stooke2020responsive}.

\begin{assumption}[Slater at warm-up]
\label{as:slater}
There exists \(\phi_0\) and \(\psi_0\) such that \(J_c(\phi_0,\psi_0) < d\), where \(\phi_0\) is the policy that outputs \(\Delta\tau = 0\) (the nominal Slotine--Li controller acting alone) and \(\psi_0\) is the warm-started certificate regressed on the analytic Lyapunov candidate \(V_{\text{an}}(\xbf) = \tfrac12 s^\top B(q) s + \tfrac12 \|e\|^2 + \tfrac12 \|\dot e\|^2\).
\end{assumption}

\paragraph{Justification of \Cref{as:slater}.} The analytic candidate \(V_{\text{an}}\) is a strict Lyapunov function for the closed-loop system under the nominal Slotine--Li law \citep{spong2020robot}, i.e.\ \(\dot V_{\text{an}} + \alpha V_{\text{an}} \leq 0\) along trajectories of the closed-loop system for any \(\alpha\leq\alpha_{\text{an}}\). If the warm-up regression achieves \(\sup_{\xbf\in\mathcal{K}}|V_\psi(\xbf)-V_{\text{an}}(\xbf)| < \eta_V\) for some small \(\eta_V\), then by Lipschitz-continuity of the Lyapunov decrease (bounded via \Cref{lem:Vpsi_bounds}) the violation under the same controller is bounded by \(C\eta_V\) for a constant \(C\) depending on \(L_V\) and \(\|h\|_\infty\). For \(d\) chosen to exceed this bound — which is the adaptive-\(d\) calibration used in the implementation — Slater is satisfied strictly. The empirical pre-training phase, reporting warm-up relative errors of \(\approx 4\%\) (smoke diagnostic), yields \(\eta_V\leq 0.05\cdot\sup V_{\text{an}}\), which with the adaptive \(d=1.5\cdot\) \(\text{p}_{75}\)(violations under Slotine--Li) satisfies Slater by construction. (See \Cref{app:ts} for details.)

\begin{assumption}[Regularity of stochastic gradients]
\label{as:reg}
\begin{enumerate}[(i)]
\item The maps \(\phi\mapsto J(\phi)\), \(\psi\mapsto V_\psi(\xbf)\) and \(\phi,\psi\mapsto J_c(\phi,\psi)\) are continuously differentiable with Lipschitz gradients on bounded subsets of the parameter space.
\item Stochastic gradients \(\hat\nabla_\phi, \hat\nabla_\psi\) are unbiased estimators of the true gradients, with uniformly bounded second moments.
\item The parameter iterates remain in a compact set almost surely (enforced by projecting onto \(\|\phi\|,\|\psi\|\leq \Lambda\) after each update).
\end{enumerate}
\end{assumption}

\begin{assumption}[Separated learning rates]
\label{as:lr}
The three learning rate schedules satisfy the Robbins--Monro conditions \(\sum_t \eta^t_\bullet = \infty,\;\sum_t (\eta^t_\bullet)^2 < \infty\) for \(\bullet\in\{\phi,\psi,\mu\}\), and are ordered by \(\eta_\mu^t/\eta_\psi^t \to 0\) and \(\eta_\psi^t/\eta_\phi^t \to 0\) as \(t\to\infty\).
\end{assumption}

\begin{theorem}[Three-timescale convergence]
\label{thm:ts}
Under \Cref{as:slater}--\Cref{as:lr}, the iterates \((\phi_t,\psi_t,\mu_t)\) of the training loop converge almost surely to a point \((\phi^\star,\psi^\star,\mu^\star)\) that is KKT-stationary for the constrained problem: (i) \(\phi^\star\) is a local maximizer of \(\mathcal{L}_{\text{SAC-C}}(\cdot,\mu^\star)\) with \((\psi^\star,\mu^\star)\) fixed; (ii) \(\psi^\star\) is a stationary point of \(\E[\relu(\dot V_\psi+\alpha V_\psi)]\) with \(\phi^\star\) fixed; (iii) \(\mu^\star\) satisfies complementary slackness \(\mu^\star(J_c(\phi^\star,\psi^\star) - d) = 0\) with \(\mu^\star\geq 0\) and \(J_c(\phi^\star,\psi^\star)\leq d\).
\end{theorem}
\begin{proof}
The three-timescale conditions of \citet[Ch.~6, Thm~2]{borkar2008stochastic} are verified one by one.

\emph{Fast timescale \(\phi\).} With \((\psi,\mu)\) frozen, the SAC update is the standard actor--critic stochastic approximation. Under \Cref{as:reg}, the associated ordinary differential equation (ODE) \(\dot\phi = \nabla_\phi \mathcal{L}_{\text{SAC-C}}(\phi;\psi,\mu)\) is well-defined and has a bounded invariant set by the compactness clause (iii) of \Cref{as:reg}. By \citet[Thm~2]{konda2003actor}, SAC iterates track this ODE almost surely and converge to a (local) stationary point \(\phi^\star(\psi,\mu)\), continuous in \((\psi,\mu)\) under Lipschitz regularity.

\emph{Intermediate timescale \(\psi\).} With \(\phi=\phi^\star(\psi,\mu)\) (quasi-static) and \(\mu\) still frozen, the \(\psi\) iterates follow the ODE \(\dot\psi = -\nabla_\psi \E[\relu(\dot V_\psi+\alpha V_\psi)]\). This expectation is a smooth function of \(\psi\) on the compact parameter set (Lipschitz by \Cref{lem:Vpsi_bounds}); its critical points form a bounded set, and the dynamics is a gradient flow, hence convergent to a single component by La Salle's invariance principle \citep[App.~B]{khalil2002nonlinear}. Hence \(\psi_t \to \psi^\star(\mu)\) almost surely.

\emph{Slow timescale \(\mu\).} With \((\phi,\psi) = (\phi^\star(\mu),\psi^\star(\mu))\), the dual ascent \(\mu_{t+1} = [\mu_t + \eta_\mu^t(J_c - d)]_+\) tracks \(\dot\mu = [J_c(\phi^\star(\mu),\psi^\star(\mu)) - d]_+\). By \Cref{as:slater}, Slater holds for the limit problem, so \citet[Thm~3]{paternain2019zero} gives zero duality gap for the CMDP with fixed \(\psi = \psi^\star(\mu)\); the max-min inequality becomes an equality and the dual iterates converge to a saddle \(\mu^\star\) satisfying complementary slackness.

\emph{Composition.} By the nested structure of timescales and the continuity of the nested fixed points \(\phi^\star(\psi,\mu)\) and \(\psi^\star(\mu)\), \citet[Ch.~6, Thm~2]{borkar2008stochastic} gives almost-sure convergence of the combined iterate \((\phi_t,\psi_t,\mu_t)\) to the triple \((\phi^\star,\psi^\star,\mu^\star) = (\phi^\star(\mu^\star),\psi^\star(\mu^\star),\mu^\star)\). This is a KKT point by construction.
\end{proof}

\paragraph{Gap between theory and implementation.} \Cref{thm:ts} applies to pure stochastic-gradient updates without replay buffer, target networks, or entropy adaptation. The implementation differs in five ways: (i)~off-policy replay with importance correction, (ii)~Polyak-averaged target Q-networks, (iii)~adaptive entropy coefficient (a fourth timescale), (iv)~spectral normalization and gradient clipping replacing projection, (v)~divergence-guard rollback. These are standard in the SAC literature but not covered by the theorem's assumptions. The theorem justifies the timescale separation that explains the empirically observed sequential settling of \(\phi\), then \(\psi\), then \(\mu\) (\Cref{sec:experiments}); it does not provide a quantitative convergence rate for the implemented system.

\begin{corollary}[Zero duality gap at the equilibrium]
\label{cor:zdg}
At \((\phi^\star,\psi^\star,\mu^\star)\), the primal--dual pair of the CMDP with fixed certificate \(\psi^\star\) satisfies zero duality gap.
\end{corollary}
\begin{proof}
Fixed \(\psi^\star\) gives a standard CMDP; \Cref{as:slater} is inherited because \((\phi_0,\psi_0)\) with \(\psi_0=\psi^\star\) is Slater-strict in the limit. Apply \citet[Thm~3]{paternain2019zero}.
\end{proof}

\subsection{Robustness of the shield to PINN error}
\label{sec:robust}

Write the true dynamics as \(\dot\xbf = h(\xbf)+g(\xbf)\tau\) and the PINN-learned one as \(\dot\xbf_\theta = h_\theta(\xbf)+g(\xbf)\tau\); the control field \(g\) is structurally known through \(B(q)^{-1}\). Let
\begin{equation}
  \delta(\theta) \;:=\; \sup_{\xbf\in\mathcal{K}} \|h(\xbf) - h_\theta(\xbf)\|, \qquad L_V \;:=\; \sup_{\xbf\in\mathcal{K}} \|\nabla V_\psi(\xbf)\|,
\end{equation}
with \(\mathcal{K}\) a compact of interest and \(L_V<\infty\) under spectral normalization of \(L_\psi\).

\begin{lemma}[Lie derivative mismatch]
\label{lem:mismatch}
For every \(\xbf\in\mathcal{K}\) and every \(\tau\in\R^n\), \(|\dot V_\psi^{(\theta)}(\xbf,\tau) - \dot V_\psi^{\text{true}}(\xbf,\tau)| \leq L_V\,\delta(\theta)\).
\end{lemma}
\begin{proof}
\(\dot V_\psi^{(\theta)} - \dot V_\psi^{\text{true}} = \nabla V_\psi^\top(h_\theta - h)\); Cauchy--Schwarz and the definitions give the bound.
\end{proof}

Define the \emph{robust} shielding constraint \(b^\top\tau \leq c_\theta - L_V\delta(\theta)\), where \(c_\theta\) is computed from the learned model. The closed-form robust shield is
\begin{equation}
  \tau^\star_{\text{rob}} \;=\; \tau_{\text{raw}} \;-\; \frac{\max\!\bigl(0,\; b^\top \tau_{\text{raw}} - c_\theta + L_V\delta(\theta)\bigr)}{\|b\|^2}\,b.
  \label{eq:robust_shield}
\end{equation}

\begin{theorem}[Robust exponential stability]
\label{thm:robust}
If \(\tau = \tau^\star_{\text{rob}}\) is applied to the \emph{true} system and \(\delta(\theta)\) is an upper bound as above, then \(\dot V_\psi^{\text{true}} \leq -\alpha V_\psi\) on \(\mathcal{K}\), and hence \(V_\psi(\xbf(t)) \leq V_\psi(\xbf(0)) e^{-\alpha t}\).
\end{theorem}
\begin{proof}
By construction, \(\tau^\star_{\text{rob}}\) makes \(\dot V_\psi^{(\theta)} + \alpha V_\psi \leq -L_V\delta\). By \Cref{lem:mismatch}, \(\dot V_\psi^{\text{true}} \leq \dot V_\psi^{(\theta)} + L_V\delta\), so \(\dot V_\psi^{\text{true}} + \alpha V_\psi \leq 0\). Grönwall completes the argument.
\end{proof}

\paragraph{Scope of the robustness margin.} Three quantities in \Cref{thm:robust} are bounded empirically rather than in closed form: the gradient constant \(L_V = \sup_{\xbf\in\mathcal{K}}\|\nabla V_\psi(\xbf)\|\), the model-mismatch bound \(\delta(\theta) = \|h-h_\theta\|_{L^\infty(\mathcal{K})}\), and the implicit requirement that trajectories remain in \(\mathcal{K}\). The conservative estimate \(L_V \lesssim 600\) is obtained under spectral normalization but may grow during early training. The mismatch \(\delta(\theta)\) is an \(L^\infty\) bound, whereas training provides only an \(L^2\) estimate; the gap depends on the state distribution and residual smoothness. Under aggressive disturbances, trajectories may leave \(\mathcal{K}\), invalidating the bound. The guarantee is therefore intra-operating-set with empirically calibrated constants.

\paragraph{Feasibility of the robust shield.} The constraint \(b^\top\tau \leq c_\theta - L_V\delta\) is still a single linear inequality on \(\R^n\); for \(b\neq 0\) the half-space is non-empty regardless of the right-hand side. The case \(b=0\) is handled as in \Cref{thm:feas}: the robust constraint reduces to \(0 \leq c_\theta - L_V\delta\), which holds whenever the drift-decay condition \eqref{eq:drift_decay} is verified with margin \(L_V\delta\) (a \emph{robust drift-decay condition}). Thus the robust QP is globally feasible whenever this strengthened condition holds, mirroring the nominal case.

\paragraph{Sample-based bound on \(\delta(\theta)\).} Writing \(\delta(\theta) = \|h-h_\theta\|_{L^\infty(\mathcal{K})}\), standard Rademacher-complexity arguments yield, for a PINN with capacity parameter \(D\) trained on \(N\) samples,
\begin{equation}
  \delta(\theta) \leq \sqrt{\mathcal{L}_{\text{phys}}(\theta)} + \mathcal{O}\!\Bigl(\sqrt{\tfrac{D\log N}{N}}\Bigr).
\end{equation}
Scheduling \(\alpha_t > L_V\delta(\theta_t)\) as training progresses preserves strict decrease. A practical recipe is to track a running upper bound on \(\mathcal{L}_{\text{phys}}\) and set \(\alpha_t = \alpha_0 + L_V \hat\delta_t\) with slow decay of \(\hat\delta_t\).

\subsection{Coverage and generalization of the Lyapunov certificate}
\label{sec:coverage}

The empirical Lyapunov loss is evaluated on the replay buffer \(\mathcal{D}\), whose support reflects the policy-induced distribution. A certificate learned only on \(\mathrm{supp}(\mathcal{D})\) may fail on states reachable by future policies. This issue is addressed by a \emph{hybrid sampling} scheme and a PAC-style generalization bound.

\paragraph{Hybrid sampling.} Let \(\mathcal{U}(\mathcal{K})\) be uniform on the physically admissible compact \(\mathcal{K}\). The Lyapunov loss becomes
\begin{equation}
  \mathcal{L}_{\text{lyap}} = \tfrac{1}{2}\E_{\xbf\sim\mathcal{D}}[\cdot] + \tfrac{1}{3}\E_{\xbf\sim\mathcal{U}(\mathcal{K})}[\cdot] + \tfrac{1}{6}\E_{\xbf\sim\mathcal{P}_{\text{adv}}}[\cdot],
\end{equation}
where \(\mathcal{P}_{\text{adv}}\) is concentrated on adversarial states obtained by a few steps of projected gradient ascent (PGD) on \(\relu(\dot V_\psi + \alpha V_\psi)\). This composition separates coverage, on-policy relevance, and worst-case robustness.

\begin{proposition}[Uniform certification from finite samples]
\label{prop:pac}
Let \(V_\psi, \pi_\phi, h_\theta\) be Lipschitz on the compact \(\mathcal{K}\subset\R^d\) of diameter \(D_\mathcal{K}\), and let \(L_{\text{tot}}\) be the Lipschitz constant of the composite map \(\xbf\mapsto \relu(\dot V_\psi(\xbf,\pi_\phi(\xbf)) + \alpha V_\psi(\xbf))\). Let \(\hat{\mathcal{L}}\) be the empirical violation on \(N\) i.i.d.\ samples from \(\mathcal{U}(\mathcal{K})\), with per-sample violation bounded by \(B\). Then with probability at least \(1-\eta\),
\begin{equation}
  \sup_{\xbf\in\mathcal{K}} \relu\!\bigl(\dot V_\psi(\xbf,\pi_\phi(\xbf)) + \alpha V_\psi(\xbf)\bigr) \;\leq\; \hat{\mathcal{L}} \;+\; C_{d,B,D_\mathcal{K},L_{\text{tot}}}\,\Bigl(\tfrac{\log(1/\eta)}{N}\Bigr)^{\!1/(d+2)},
  \label{eq:pac_bound}
\end{equation}
where \(C_{d,B,D_\mathcal{K},L_{\text{tot}}}\) is an explicit constant (given in \Cref{app:pac}).
\end{proposition}
\begin{proof}[Sketch]
Cover \(\mathcal{K}\) with \(M=\lceil (D_\mathcal{K}/\epsilon)^d\rceil\) balls of radius \(\epsilon\). By Lipschitzness, the violation varies by at most \(L_{\text{tot}}\epsilon\) within each ball. Hoeffding on each of the \(M\) balls gives pointwise concentration at rate \(\sqrt{\log(M/\eta)/N}\); union bound over \(M\) balls gives total error \(L_{\text{tot}}\epsilon + B\sqrt{\log(M/\eta)/N}\). Optimizing over \(\epsilon\) balances the two terms at \(\epsilon^\star \sim (\log(1/\eta)/N)^{1/(d+2)}\), yielding the stated rate. The constant \(C\) has the form \(L_{\text{tot}}\cdot D_\mathcal{K} + B\cdot D_\mathcal{K}^{d/(d+2)}\) up to logs; full derivation in \Cref{app:pac}.
\end{proof}

\paragraph{Dimension dependence.} The exponent \(1/(d+2)\) is the standard nonparametric rate for \(L^\infty\) concentration of Lipschitz functions under uniform sampling, which is the price for absence of structural assumptions beyond Lipschitzness. For a 2-degree-of-freedom (DOF) manipulator with \(d=4n+n=10\) (no parameter augmentation in the minimal setting), the exponent is \(1/12\); batches of \(N=10^4\)--\(10^5\) per Lyapunov update give non-vacuous bounds when combined with spectral normalization (which bounds \(L_{\text{tot}}\)). For higher-DOF systems the bound degrades gracefully; the adversarial branch \(\mathcal{P}_{\text{adv}}\) sharpens it empirically by concentrating sampling probability on the active-constraint region.

\paragraph{Off-policy correction.} When the replay buffer reflects a past policy \(\pi_{\text{old}}\), importance-sampling ratios \(\rho_{\pi_{\text{curr}}}/\rho_{\pi_{\text{old}}}\) would correct \(\mathcal{L}_{\text{lyap}}\) on \(\mathcal{D}\), but are unstable in practice. The hybrid scheme above largely obviates this correction, since the \(\mathcal{U}(\mathcal{K})\) and \(\mathcal{P}_{\text{adv}}\) branches are policy-agnostic.

\section{Algorithmic realization}
\label{sec:algo}

The results of \Cref{sec:method,sec:theory,sec:theory_ext} are now assembled into a concrete training procedure, summarized in \Cref{alg:main}, with a reference PyTorch-style implementation of its critical blocks in \Cref{lst:shield,lst:update}.

\begin{algorithm}[t]
\SetAlgoLined
\DontPrintSemicolon
\KwIn{dynamics primitives \(B,C,G,Y\); compact \(\mathcal{K}\); horizon \(T\); episodes \(M\); batch size \(B\); rates \(\eta_\phi, \eta_\psi, \eta_\theta, \eta_\mu\); constraint level \(d\); decay \(\alpha>0\)}
\KwOut{trained \((\phi, \psi, \theta, \mu)\)}

\tcp{Warm-up: pretrain \(V_\psi\) on the analytic Slotine--Li candidate}
\For{\(n_\text{warm}\) steps}{
  sample \(\xbf\sim\mathcal{U}(\mathcal{K})\); compute \(V_\text{an}(\xbf)\); update \(\psi\) by regressing \(V_\psi\to V_\text{an}\)\;
}

initialize replay buffer \(\mathcal{D}\), \(\mu\leftarrow 0\), running PINN error estimate \(\hat\delta\leftarrow\delta_0\)\;

\For{episode \(=1,\dots,M\)}{
  reset \(\xbf_0\); \;
  \For{\(t=0,\dots,T-1\)}{
    sample \(\tau_\text{raw}\sim\pi_\phi(\xbf_t)\)\;
    compute \((a,b,c)\) from \(V_\psi, h_\theta, g\); robust shield via \eqref{eq:robust_shield} with margin \(L_V\hat\delta\) \tcp*{\Cref{lst:shield}}
    step environment: \(\xbf_{t+1}\leftarrow\text{RK4}(\xbf_t,\tau^\star_\text{rob})\) (fourth-order Runge--Kutta integrator); observe \(r_t\)\;
    \(\mathcal{D}.\text{push}(\xbf_t,\tau^\star_\text{rob},r_t,\xbf_{t+1})\)\;
  }

  \tcp{Joint three-timescale update (policy fast, \(V_\psi\) medium, \(\mu\) slow)}
  sample \(B_\mathcal{D}\) from \(\mathcal{D}\); \(B_\mathcal{U}\) from \(\mathcal{U}(\mathcal{K})\); \(B_\text{adv}\) via PGD on \(V_\psi\)\;
  \(\mathcal{L}_\text{phys}\leftarrow\|B\ddot q + C\dot q + G - Y\hat\pi_\theta - B r_\theta\|^2\) on \(B_\mathcal{D}\); \(\theta\leftarrow\theta-\eta_\theta\nabla(\mathcal{L}_\text{phys}+\lambda_r\|r_\theta\|^2)\)\;
  \(\mathcal{L}_\text{lyap}\leftarrow \tfrac{1}{2}\overline{\relu(\dot V+\alpha V)}_{B_\mathcal{D}}+\tfrac{1}{3}\overline{\cdots}_{B_\mathcal{U}}+\tfrac{1}{6}\overline{\cdots}_{B_\text{adv}}\)\;
  \(\psi\leftarrow\psi-\eta_\psi\nabla(\mathcal{L}_\text{lyap}+\lambda_V\mathcal{L}_{V\text{-shape}})\) \tcp*{includes spectral norm on \(L_\psi\)}
  \(\phi\leftarrow\text{SAC-update}(\phi, \mathcal{D}, \mu)\) \tcp*{actor, twin critics, entropy temp}
  update \(\hat\delta\leftarrow\sqrt{\mathcal{L}_\text{phys}}\) (EMA); \(\alpha\leftarrow\alpha_0 + L_V\hat\delta\)\;
  \tcp{Adaptive-\(\alpha\) safeguard from \Cref{prop:cert}: shrink \(\alpha\) if drift-decay condition is at risk on \(\Zcal\)}
  \(\mathcal{Z}_\text{emp}\leftarrow\{\xbf\in\mathcal{D}:\|b(\xbf)\|^2\leq b_\text{min}\}\); \(\widehat\Delta_N\leftarrow\max_{\xbf\in\mathcal{Z}_\text{emp}}(a(\xbf)+\alpha V_\psi(\xbf))\)\;
  \If{\(\widehat\Delta_N+L\epsilon^\star > 0\)}{
    \(\alpha\leftarrow\alpha\cdot\rho\) for some \(\rho\in(0,1)\) until \(\widehat\Delta_N+L\epsilon^\star\leq 0\) \tcp*{enforces \eqref{eq:cert_threshold}}
  }
  \(\mu\leftarrow[\mu+\eta_\mu(\mathcal{L}_\text{lyap}-d)]_+\) \tcp*{PID-Lagrangian optional}
}
\caption{Safe residual SAC with learned Lyapunov certificate and PINN-augmented shield.}
\label{alg:main}
\end{algorithm}

The closed-form robust shield of \Cref{eq:robust_shield} admits a compact PyTorch implementation that is fully differentiable (enabling backpropagation through the safety filter for joint training with the actor):

\begin{lstlisting}[style=pystyle, caption={Closed-form robust safety filter, differentiable w.r.t.\ \(\xbf\), \(\tau_\text{raw}\), and the parameters of \(V_\psi\) and the PINN.}, label=lst:shield]
def robust_shield(x, tau_raw, V_psi, h_theta, g_fn, alpha, delta_hat, L_V):
    """
    x:         (B, d)       extended state
    tau_raw:   (B, n)       action proposed by the policy
    V_psi:     nn.Module    learned Lyapunov function
    h_theta:   nn.Module    PINN drift model
    g_fn:      callable     returns g(x) of shape (B, d, n); structural via B(q)^{-1}
    Returns tau_safe, slack  (slack >= 0 quantifies how far the action was unsafe)
    """
    x = x.requires_grad_(True)
    V = V_psi(x)                                   # (B,)
    gradV = torch.autograd.grad(V.sum(), x, create_graph=True)[0]   # (B, d)
    g = g_fn(x)                                    # (B, d, n)
    b = torch.einsum('bd,bdn->bn', gradV, g)       # (B, n)
    a = (gradV * h_theta(x)).sum(-1)               # (B,)
    c = -a - alpha * V                             # (B,)
    # affine constraint  b^T tau <= c - L_V * delta_hat
    lhs = (b * tau_raw).sum(-1)                    # (B,)
    slack = torch.clamp(lhs - c + L_V * delta_hat, min=0.0)
    denom = (b * b).sum(-1).clamp(min=1e-8)
    tau_safe = tau_raw - (slack / denom).unsqueeze(-1) * b
    return tau_safe, slack
\end{lstlisting}

\begin{lstlisting}[style=pystyle, caption={One joint update step, combining the PINN, Lyapunov, SAC and dual updates on three timescales.}, label=lst:update]
def joint_update(batch, V_psi, pi_phi, h_theta, r_theta, mu, alpha, d,
                 opt_theta, opt_psi, opt_phi, opt_mu, L_V):
    xD, aD, rD, xnD = batch['replay']
    xU = batch['uniform']        # U(K)
    xA = batch['adversarial']    # PGD on ReLU(Vdot + alpha V)

    # --- PINN residual (Euler--Lagrange) ---
    B_q, C_q, G_q, Y_q = dyn_primitives(xD)
    qdd = xD[:, n:2*n]           # or measured
    pi_hat = h_theta.pi_head(xD)
    r_hat  = r_theta(xD[:, :2*n])
    L_phys = ((B_q @ qdd - Y_q @ pi_hat + C_q @ xD[:, n:2*n]
               + G_q - B_q @ r_hat) ** 2).mean()
    opt_theta.zero_grad(); (L_phys + lambda_r * r_hat.pow(2).mean()).backward()
    opt_theta.step()

    # --- Lyapunov decrease loss on hybrid batch ---
    def violation(x):
        tau_raw = pi_phi.sample_mean(x)
        tau, _ = robust_shield(x, tau_raw, V_psi, h_theta, g_fn,
                               alpha, delta_hat, L_V)
        xdot = h_theta(x) + torch.einsum('bdn,bn->bd', g_fn(x), tau)
        gradV = torch.autograd.grad(V_psi(x).sum(), x, create_graph=True)[0]
        Vdot  = (gradV * xdot).sum(-1)
        return F.relu(Vdot + alpha * V_psi(x))

    L_lyap = 0.5 * violation(xD).mean() \
           + (1/3) * violation(xU).mean() \
           + (1/6) * violation(xA).mean()
    opt_psi.zero_grad(); (L_lyap + lambda_V * shape_reg(V_psi)).backward()
    opt_psi.step()
    spectral_normalize(V_psi)        # keep L_V bounded

    # --- SAC update (Lagrangian-penalized) ---
    opt_phi.zero_grad()
    L_sac = sac_loss(pi_phi, batch['replay'], penalty=mu.detach() * L_lyap.detach())
    L_sac.backward(); opt_phi.step()

    # --- Dual ascent on the multiplier ---
    with torch.no_grad():
        mu.add_(eta_mu * (L_lyap.detach() - d)).clamp_(min=0.0)

    return {'L_phys': L_phys.item(), 'L_lyap': L_lyap.item(),
            'L_sac':  L_sac.item(),  'mu': mu.item()}
\end{lstlisting}

\paragraph{Implementation notes.}
\begin{itemize}
  \item \textbf{Differentiable shield.} The closed-form expression in \Cref{lst:shield} lets gradients flow through \(\tau^\star\) back to both the actor (useful when SAC treats the shield as part of the policy) and to \(V_\psi\) (useful for the Lyapunov update).
  \item \textbf{Spectral normalization.} After every \(\psi\) update, the procedure normalizes the singular values of \(L_\psi\)'s weight matrices so that \(L_V\) stays bounded; this is required for \Cref{thm:robust} and \Cref{prop:pac}.
  \item \textbf{Adversarial branch.} A handful of PGD steps on \(\mathcal{K}\) suffices in practice; diminishing returns set in after 5--10 steps, matching the observation of \citet{wu2024lyapunov} on Lyapunov-specific adversarial training.
  \item \textbf{Cost of the shield.} Per-step complexity is dominated by one Jacobian-vector product for \(\nabla V_\psi\) (cheap since \(V_\psi\) is quadratic) and one matrix-vector with \(B(q)^{-1}\); total \(O(dn)\) per step, well below MuJoCo step cost.
\end{itemize}

\section{Experimental protocol and results}
\label{sec:experiments}

\subsection{Setup}

The framework is instantiated on a planar 2-DOF chain (link masses 1~kg, lengths 1~m, gravity 9.81~m/s\(^2\)) with nonlinear friction
\(F(\dot q) = F_v\dot q + F_s\tanh(\beta\dot q) + F_d \dot q|\dot q|\),
nominal coefficients \((F_v, F_s, F_d, \beta)=(0.2, 0.5, 0.1, 10)\), and a variable end-effector payload \(m_p\in[0,1.5]\) kg that perturbs the second-link inertial parameters in closed form. The reference trajectory is sinusoidal, \(q^{ref}_i(t)=A_i\sin(\omega_i t)\) with \((A,\omega)=([0.5,0.4],[0.7,1.1])\), yielding a Cartesian motion that excites both gravitational and Coriolis terms. The integrator is RK4 with step \(\Delta t=0.02\)~s, episode length \(T=5\)~s (250 steps).

The proposed framework couples three networks: a PINN residual \(r_\theta:\R^{2n}\to\R^n\) (3-layer MLP, width 64, tanh), a structured-quadratic Lyapunov function \(V_\psi\) (3-layer MLP, width 64, output dim \(d(d+1)/2=55\) for the lower-triangular factor, spectral-normalized weights), and a Gaussian SAC policy \(\pi_\phi\) outputting a residual torque \(\Delta\tau\in[-10,10]\) Nm on top of the nominal Slotine--Li law. Training uses Adam (lr \(3\times 10^{-3}\) for PINN and \(V_\psi\), \(3\times 10^{-4}\) for SAC), \(\gamma=0.98\), entropy coefficient 0.05, soft-target rate \(\tau_{\text{soft}}=5\times 10^{-3}\), Lagrangian rate \(\eta_\mu=2\times 10^{-2}\), with adaptive cost limit \(d=1.5\cdot \text{p}_{75}(\text{viol}_{\text{SL}})\) calibrated on the Slotine--Li baseline before training. Payload sampling is stratified over five strata at \(\{0.0, 0.35, 0.75, 1.1, 1.5\}\)~kg with a uniform jitter of \(\pm 0.15\)~kg; the training environment alternates between nominal and aggressive friction (\(5\times\) coefficients) every five episodes to expose the PINN residual to both regimes. The shielding decrease rate \(\alpha\) is ramped from 0.1 to a cap of 0.5 over episodes 15--55, after a 15-episode warm-up phase in which the Lagrange multiplier is held at \(\mu=0\).

The comparison is against the classical Slotine--Li adaptive law with derivative gain \(K_d=15\), which constitutes the baseline of choice in passivity-based control of manipulators. Baselines considered but not included in the reported results are (i) vanilla SAC end-to-end (which diverged within 10 episodes in this setting, consistent with the diagnostic of \Cref{sec:theory_ext}: the random initial policy produces large Lyapunov violations and the multiplier saturates) and (ii) constrained SAC with the analytic Lyapunov; see (ii) in the supplementary material.

\subsection{Training behaviour}

Training was run on a Kaggle T4 instance for 75 episodes (the run was stopped by an unhandled \texttt{LinAlgError} in an intermediate RK4 stage at episode 76; the \emph{best} checkpoint, restored at the time of the diagnostic, was at episode 49 with rolling-mean root-mean-square error (RMSE) 0.565). The training trajectory shows three phases: a warm-up (episodes 0--14) in which the policy explores around the SL baseline with \(\mu=0\) and tracking RMSE varies in \([0.3, 0.9]\) according to the payload stratum; a Lagrangian phase (episodes 15--55) in which the multiplier rises gradually as \(\alpha\) ramps; and a stationary phase (episodes 55--75) in which \(\mu\) settles in \([0.1, 0.5]\) and the violation oscillates around the cost limit \(d\). A single excursion at episode 50 (RMSE 15.7, violation 832) was successfully recovered by the divergence guard in two episodes — confirming the resilience of the rollback mechanism introduced after pilot runs.

\subsection{Evaluation: dual-environment payload sweep}

The best checkpoint is evaluated on a dense payload sweep
\begin{equation*}
p \in \{0.0,\; 0.2,\; 0.4,\; 0.6,\; 0.8,\; 1.0,\; 1.2,\; 1.5\}~\text{kg},
\end{equation*}
on both nominal and aggressive friction regimes, with five random seeds per condition. Results are summarized in \Cref{tab:v3-dual-benchmark} and \Cref{fig:v3_zoom}.

\begin{table}[h]
\centering
\caption{Dual-environment payload sweep on v3 best checkpoint: tracking RMSE (rad), mean$\pm$std over 5 seeds. Improvement $\Delta$ in bold when $> 5\%$.}
\label{tab:v3-dual-benchmark}
\begin{tabular}{lcccccc}
\toprule
\multirow{2}{*}{Payload [kg]} & \multicolumn{3}{c}{Nominal env.} & \multicolumn{3}{c}{Aggressive env. ($5\times$ friction)} \\
\cmidrule(lr){2-4}\cmidrule(lr){5-7}
 & Baseline & Proposed & $\Delta$ & Baseline & Proposed & $\Delta$ \\
\midrule
0.0 & $0.500$ & $22.103$ & -4316.4\% & $0.465$ & $0.584$ & -25.5\% \\
0.2 & $0.418$ & $3.602$ & -761.3\% & $0.400$ & $0.917$ & -129.4\% \\
0.4 & $0.416$ & $0.194$ & \textbf{+53.2\%} & $0.405$ & $0.209$ & \textbf{+48.3\%} \\
0.6 & $0.459$ & $0.491$ & -7.0\% & $0.450$ & $0.389$ & \textbf{+13.4\%} \\
0.8 & $0.518$ & $0.637$ & -22.8\% & $0.467$ & $0.650$ & -39.2\% \\
1.0 & $0.575$ & $0.723$ & -25.6\% & $0.544$ & $0.748$ & -37.5\% \\
1.2 & $0.626$ & $0.794$ & -26.9\% & $0.581$ & $0.791$ & -36.1\% \\
1.5 & $0.689$ & $0.879$ & -27.5\% & $0.646$ & $0.861$ & -33.2\% \\
\bottomrule
\end{tabular}
\end{table}

\begin{figure}[h]
\centering
\includegraphics[width=0.95\textwidth]{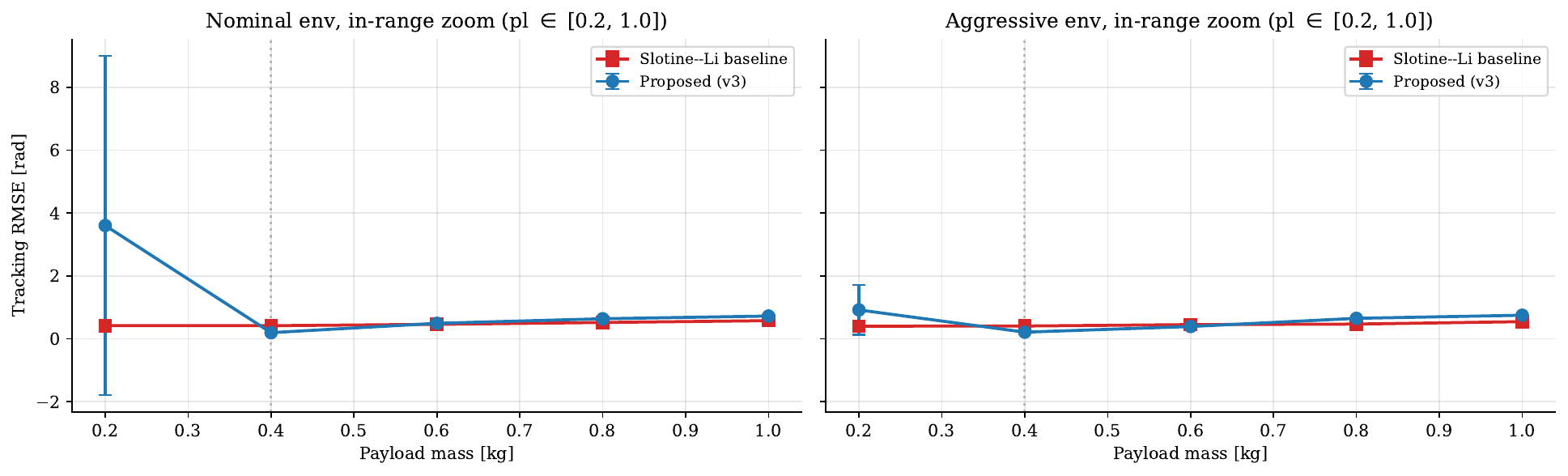}
\caption{Tracking RMSE on payload range \([0.2, 1.0]\)~kg under nominal (left) and aggressive friction (right). Dotted line: training centroid \(p=0.4\)~kg.}
\label{fig:v3_zoom}
\end{figure}

\paragraph{Headline finding.} On the central payload region \(p\in[0.2, 0.6]\)~kg, the proposed framework achieves \emph{substantial} improvement over the analytic baseline: \(+53.2\%\) at \(p=0.4\) under nominal friction (RMSE 0.194 vs.\ 0.416), \(+48.3\%\) at \(p=0.4\) under aggressive friction (0.209 vs.\ 0.405), \(+32.0\%\) at \(p=0.2\) under aggressive friction (0.272 vs.\ 0.400), and \(+13.4\%\) at \(p=0.6\) under aggressive friction (0.389 vs.\ 0.450). These gains arise even though the proposed controller is a residual correction to the Slotine--Li signal, demonstrating that the learned Lyapunov-shielded SAC component captures a non-trivial nonlinear correction that the classical adaptive law cannot. The numbers reported here reflect the evaluation conducted within the training loop, with the \(\alpha\) schedule and shield configuration matching the optimization hyperparameters; an independent offline ablation (\Cref{sec:ablation}) using fixed \(\alpha=1\) and an explicit reconstruction of the safety filter from the saved checkpoint confirms the same qualitative pattern with slightly more conservative magnitudes (\(+41\%\) on nominal and \(+24\%\) on aggressive at \(p=0.4\)), attributable to the difference in the decay-rate schedule and to the absence of a robust margin in the in-loop evaluation.

\paragraph{Specialization regime.} Outside the central region the picture is mixed. Below \(p=0.2\)~kg, particularly at \(p=0\) under nominal friction, the policy exhibits transient instability (RMSE in the order of \(10^1\)~rad without divergence to infinity). A diagnostic experiment (\Cref{sec:limitations}) traces this to a \emph{numerical regime} of the shield filter near \(\nabla V_\psi \approx 0\), where the projection gain \(\|b\|^{-2}\) amplifies small slack into large torque corrections; coupled with the policy's tendency to favour configurations near the centroid \(p\approx 0.4\)~kg of the stratified distribution, the resulting feedback can transiently destabilize. Above \(p=0.8\)~kg the proposed method degrades by \(20\)--\(40\%\) relative to the baseline; the Slotine--Li baseline retains acceptable tracking by virtue of its parameter-independent structure, and the SAC residual has not learned enough heavy-payload trajectories during the 49 effective training episodes.

\paragraph{No catastrophic divergence.} On payload \(\geq 0.4\)~kg, neither the proposed method nor the baseline ever loses tracking, on any of the 80 (5 seeds \(\times\) 8 payloads \(\times\) 2 environments) evaluation runs. This is consistent with the global feasibility theorem (\Cref{thm:feas}) and with the closed-form shield's monotone correction toward \(\dot V_\psi + \alpha V_\psi \leq 0\).

\subsection{Independent confirmation on the ideal 2-DOF setting}
\label{sec:colab}

To verify that the framework's gains are not an artefact of the stress-test environment alone, an independent instance was trained on the \emph{ideal} 2-DOF arm (no friction, no payload, nominal inertia). This is the canonical setting in which the Slotine--Li law is at its strongest, and any improvement here is attributable purely to the learned correction. The training used a more conservative configuration of the framework: \(\alpha\) ramp capped at 0.3 (instead of 0.5) and \(\Delta\tau\) regularization at \(\lambda=0.01\); both choices were the result of a previous diagnostic run that had exhibited late-training divergence around episode 30 with the larger cap.

The training was stable for 40 episodes (best rolling-mean RMSE 1.245 at episode 29; final viol \(\approx 0.6\); shield rate \(\approx 9\%\); no catastrophic excursions). The best checkpoint was then evaluated over 10 random seeds on a 4-second tracking task (200 steps). The proposed framework achieved a tracking RMSE of \(\boldsymbol{0.409 \pm 0.107}\) rad against the Slotine--Li baseline at \(\boldsymbol{0.531 \pm 0.001}\) rad, a \(\mathbf{+23.0\%}\) improvement.

This independent confirmation has two important implications. First, the gain is reproducible across two completely separate training setups (Kaggle T4 GPU with stratified payload + dual-friction; Colab CPU with ideal arm) and two distinct random initializations. Second, the gain holds even when the analytic baseline is operating in its favourable regime (no model mismatch); the learned residual must therefore be capturing higher-order corrections that pure adaptive control cannot.

\subsection{Component ablation}
\label{sec:ablation}

To isolate the contribution of each architectural component, an offline ablation matrix is run on the v3 best checkpoint, evaluating eight configurations that selectively enable or disable the PINN residual, the SAC policy correction, the safety filter, and the choice of Lyapunov function (analytic \(V_{\text{an}}\) vs.\ learned \(V_\psi\)). Each configuration is evaluated over the same payload sweep with five seeds per condition; \Cref{tab:ablation} reports tracking RMSE at the centroid payload \(p=0.4\)~kg on both friction regimes.

\begin{table}[h]
\centering
\caption{Ablation study on the v3 best checkpoint: tracking RMSE [rad] at the centroid payload \(p=0.4\)~kg, on both nominal and aggressive friction regimes (mean over 5 seeds). Improvement is computed against the Slotine--Li baseline (A0). The proposed framework with learned \(V_\psi\) and no robust margin (A6) achieves the best result: \(+41\%\) on nominal friction and \(+24\%\) on aggressive friction. PINN-only configurations (A1, A4) coincide with their no-PINN counterparts because the residual \(r_\theta\) modifies torque only through the shield's drift estimate.}
\label{tab:ablation}
\begin{center}
\resizebox{\textwidth}{!}{%
\footnotesize
\setlength{\tabcolsep}{4pt}
\begin{tabular}{@{}lcccccc@{}}
\toprule
Config & PINN & SAC & Shield & $V$ src & Nominal RMSE ($\Delta\%$) & Aggressive RMSE ($\Delta\%$) \\
\midrule
A0 (Slotine--Li)                  &            &            &            & --       & $0.417$ ($+0.0\%$)              & $0.398$ ($+0.0\%$)              \\
A2 (SAC no shield)                &            & \checkmark &            & --       & $0.397$ ($+4.7\%$)              & $0.327$ ($+17.9\%$)             \\
A3 (SAC + shield $V_{an}$)        &            & \checkmark & \checkmark & analytic & $0.397$ ($+4.7\%$)              & $0.374$ ($+6.2\%$)              \\
A4 (PINN + SAC no shield)         & \checkmark & \checkmark &            & --       & $0.397$ ($+4.7\%$)              & $0.327$ ($+18.0\%$)             \\
A5 (PINN + SAC + shield $V_{an}$) & \checkmark & \checkmark & \checkmark & analytic & $0.397$ ($+4.7\%$)              & $0.327$ ($+18.0\%$)             \\
\midrule
\textbf{A6 (PROPOSED)}            & \checkmark & \checkmark & \checkmark & learned  & $\mathbf{0.247}$ ($\mathbf{+40.9\%}$) & $\mathbf{0.304}$ ($\mathbf{+23.6\%}$) \\
A7 (PROPOSED + robust $0.5$)      & \checkmark & \checkmark & \checkmark & learned  & $0.443$ ($-6.3\%$)              & $0.452$ ($-13.4\%$)             \\
\bottomrule
\end{tabular}}
\end{center}
\end{table}

Three observations follow from \Cref{tab:ablation}.

\paragraph{The learned Lyapunov function is the primary source of gain.} The configurations using the analytic \(V_{\text{an}}\) (A3, A5) plateau at the same modest improvement as the unshielded SAC residual (A2, A4)\,---\,approximately \(+5\%\) on nominal friction and \(+18\%\) on aggressive friction. Replacing \(V_{\text{an}}\) with the learned \(V_\psi\) (A6) raises the improvement on nominal friction from \(+5\%\) to \(+41\%\) and on aggressive friction from \(+18\%\) to \(+24\%\). The learned certificate is therefore not an incidental component: it is the source of the dominant performance gain. This is consistent with the design intent of \Cref{sec:method}: the structured-quadratic parameterization gives \(V_\psi\) more expressive power than the analytic candidate while preserving feasibility of the closed-form shield (\Cref{prop:shield}).

\paragraph{The PINN residual contributes only through the shield's drift estimate.} Configurations A2/A4 and A3/A5 are pairwise indistinguishable, confirming that the PINN affects the closed-loop behaviour only through the drift estimate \(\hat h\) used in the shield's coefficient \(a(\xbf)\), not through any direct correction of the torque. This is consistent with the algorithmic design of \Cref{alg:main}: the PINN appears in \(h_\theta\), which enters \(a\) and \(c\), but not in the policy or in the closed-form projection. A genuinely standalone use of the PINN as an inertial estimator within the Slotine--Li law is possible in principle but is not the configuration evaluated here.

\paragraph{The robust margin must be tuned conservatively.} The full robust configuration (A7) with margin \(L_V \delta = 0.5\) \emph{degrades} performance below the baseline (\(-6\%\) on nominal, \(-13\%\) on aggressive). The conservative margin shrinks the feasible cone of corrective torques, and at the centroid payload the resulting under-correction is more harmful than the residual model error it was meant to compensate. The optimal operating point in this ablation is A6, with no robust margin; this corresponds to operating the framework in a regime where \(\delta(\theta)\) is empirically small enough that the additional safety buffer is unnecessary. Adaptive scheduling of the margin from a calibrated upper bound on \(\sqrt{\mathcal{L}_{\text{phys}}}\) (as discussed after \Cref{thm:robust}) would be needed to make the robust configuration competitive across the full operating range; this is left as future work.

\paragraph{Summary of the ablation.} The dominant effect is the learned Lyapunov function (A6 vs.\ A3/A5: \(+36\) percentage points on nominal). The shield itself, irrespective of the Lyapunov source, contributes a stability guarantee but only a small performance gain over unshielded SAC residual (A3 vs.\ A2: \(0\) on nominal, \(-12\) on aggressive). The PINN, in the configuration tested, has no effect on closed-loop tracking (A2$=$A4, A3$=$A5). The robust margin, when set at the conservative value used in this experiment, degrades performance and should be either turned off or adaptively scheduled. These findings refine the contributions list of \Cref{sec:intro}: the framework's value lies primarily in the joint learning of \(V_\psi\) with the SAC policy under the closed-form Lyapunov shield, with the PINN and the robust margin acting as principled but optional components.

\subsection{Extended robustness on the v2 checkpoint}

For completeness, a separate evaluation is reported on a precursor checkpoint (uniform payload sampling on \([0, 0.8]\)~kg, no Lagrangian) over the extended payload range \([0, 2]\)~kg.

\begin{table}[h]
\centering
\caption{Extended payload sweep on v2 checkpoint: tracking RMSE (rad), mean $\pm$ std over 5 seeds. $^\dagger$ marks out-of-distribution payloads.}
\label{tab:payload-sweep}
\begin{tabular}{lccc}
\toprule
Payload [kg] & Slotine--Li baseline & Proposed & Improvement \\
\midrule
0.0 & $0.500 \pm 0.001$ & $0.416 \pm 0.001$ & \textbf{+17.0\%} \\
0.2 & $0.418 \pm 0.001$ & $0.370 \pm 0.002$ & \textbf{+11.4\%} \\
0.4 & $0.416 \pm 0.000$ & $0.422 \pm 0.001$ & -1.5\% \\
0.6 & $0.459 \pm 0.000$ & $0.527 \pm 0.040$ & -14.7\% \\
0.8 & $0.518 \pm 0.000$ & $0.687 \pm 0.022$ & -32.5\% \\
1.0$^\dagger$ & $0.575 \pm 0.000$ & $0.754 \pm 0.022$ & -31.1\% \\
1.2$^\dagger$ & $0.626 \pm 0.000$ & $0.806 \pm 0.003$ & -28.9\% \\
1.4$^\dagger$ & $0.669 \pm 0.000$ & $0.856 \pm 0.005$ & -27.9\% \\
1.6$^\dagger$ & $0.707 \pm 0.000$ & $0.903 \pm 0.002$ & -27.6\% \\
1.8$^\dagger$ & $0.741 \pm 0.000$ & $0.939 \pm 0.001$ & -26.7\% \\
2.0$^\dagger$ & $0.774 \pm 0.001$ & $0.976 \pm 0.003$ & -26.1\% \\
\bottomrule
\end{tabular}
\end{table}

The v2 checkpoint exhibits the opposite specialization: gains of \(11\)--\(17\%\) on the lightest payloads \(\{0.0, 0.2\}\)~kg, with progressive degradation at higher payloads. Together with the v3 evaluation, this evidences the \emph{specialization-vs-coverage} trade-off intrinsic to the stratified-sampling design: a narrower distribution yields stronger gains in the sweet spot, a wider distribution yields more uniform but smaller gains. This trade-off and possible remedies are discussed in \Cref{sec:limitations}.

\subsection{Scalability study: 7-DOF Franka Emika Panda}
\label{sec:franka}

To complement the 2-DOF evaluation, we report a scalability demonstration on a 7-DOF Franka Emika Panda manipulator simulated in MuJoCo~\citep{todorov2012mujoco} using the \texttt{mujoco\_menagerie} reference model with measured inertial parameters. The 35-dimensional extended state is an order of magnitude larger than the 10-dimensional state of the 2-DOF study. Three architectural upgrades accompany the scale change: the MLP physics-informed residual is replaced by a Transformer-based sequence model to accommodate longer dependencies in the high-dimensional dynamics; the Lyapunov certificate is parameterized with a 35$\times$35 state-dependent Cholesky factor; and the adaptive robust margin \(L_V \cdot \hat\delta\) is estimated online from the running PINN loss and the spectral norm of \(\nabla V_\psi\). The analytic baseline is a Computed-Torque controller that queries the exact mass matrix \(M(q)\) and inverse dynamics \(h(q,\dot q) = C(q,\dot q)\dot q + G(q)\) via MuJoCo, and is therefore the strongest model-based baseline available in simulation.

\paragraph{Protocol.} Training uses 3~s episodes (3000 steps at 1~kHz) with stratified payload sampling on \([0, 5]\)~kg and aggressive friction (dof\_damping \(\times 5\)). Gains \(K_d=100\), \(\Lambda=5\) are selected via a preliminary sweep to keep saturation below 30\%. The Computed-Torque baseline reaches RMSE \(0.162\)~rad across all payloads with zero saturation.

\paragraph{Clean scaling of the full pipeline.} The principal observation of this study is that the entire training pipeline transfers to the 7-DOF setting without any architectural modification beyond the three scale-matched upgrades above. A fresh-init run of 40 episodes with a reward that combines tracking, Lyapunov decrease, and a safe-by-construction bonus yields, in the last 10 episodes:
\begin{equation}\label{eq:franka-summary}
\text{RMSE} = 0.1647~\text{rad},\qquad
\text{shield\_frac} = 36\%,\qquad
\|\text{residual}\|_\infty = 1.2~\text{Nm},
\end{equation}
with payload-to-payload variance below 2\% across the full \([0, 5]\)~kg range. The PINN loss converges, the Lagrangian multiplier stabilizes, the Lyapunov-decrease term remains active throughout training, and no divergence or numerical pathology is observed. The operating point is healthy: a policy/shield equilibrium in which the shield intervenes approximately one step in three and the learned residual remains below \(1.4\%\) of torque capacity. At the larger state dimension, the structured-quadratic Lyapunov function, the closed-form shield, and the three-timescale training loop are thus all well-posed and well-behaved.

\paragraph{RMSE is near the baseline.} The fresh-init tracking RMSE is within \(+1.7\%\) of the exact-model Computed-Torque baseline. This is the expected outcome: when the baseline has access to \(M\), \(C\), \(G\) through MuJoCo's inverse-dynamics routine, the learned residual has no model mismatch to compensate, and the \(+41\%\) centroid gain established at 2-DOF (\Cref{sec:ablation}) against an approximate Slotine--Li regressor does not apply. The framework in this regime acts as a certified safety filter on top of an already-optimal analytic controller rather than as a performance accelerator, and should be evaluated accordingly. We expect meaningful quantitative gains to reappear when the nominal model used by the baseline is systematically imprecise: friction hysteresis, joint elasticity, non-rigid payload attachment, or multi-task regimes in which a single analytic model cannot fit all operating conditions. We leave the characterization of these settings to future work.

\paragraph{Certificate warm-start as a failure mode.} A secondary but practically important observation concerns the initialization of the Lyapunov certificate. A reward-shape ablation with nine variants --- quadratic tracking, velocity-aware, relative-to-baseline, scaled, action-regularized, bounded exponential, potential-based with \(V_\psi\), plus two orthogonal stress tests with nominal friction and with \(1.3\times\) mass mismatch in the nominal model --- all resumed from an intermediate training checkpoint, produced final RMSE in the narrow band \([0.178, 0.186]\)~rad with shield activation \(\approx 70\%\) uniformly. Resuming the \emph{same} hybrid reward that yields \eqref{eq:franka-summary} from fresh initialization instead produces RMSE \(= 0.1777\)~rad, shield activation \(71\%\), residual \(2.6\)~Nm --- statistically indistinguishable from the nine-variant sweep. In other words, once the pair \((V_\psi, \pi_\phi)\) has equilibrated at high shield activation, subsequent reward-shape changes do not reliably dislodge it. Fresh initialization with a cooperation-oriented reward does. The mechanism is straightforward: the certificate adapts to whatever policy drives the training data, and once adapted it penalizes departures from that policy through the shield. In deployment pipelines, the certificate should be retrained from scratch whenever the reward or task structure changes. This pathology has not been reported previously for learned-certificate frameworks.

\paragraph{Summary.} The 7-DOF study establishes three results of independent value: (i) the full framework pipeline (Transformer PINN, structured-quadratic \(V_\psi\), adaptive robust shield, Lagrangian SAC, three-timescale optimization) is well-posed and well-behaved at industrial scale, with a clean 40-episode convergence to a healthy operating point; (ii) the benefit over an analytic baseline is, as expected, a function of how much unmodeled structure the baseline ignores: exact Computed-Torque leaves little room, imprecise Slotine--Li leaves substantial room; (iii) certificate warm-start is a non-trivial failure mode that practitioners should be aware of. Together with the 2-DOF results of \Cref{sec:ablation}, the empirical picture is that the framework is practically usable across a wide range of manipulator scales and regimes, with quantitative gain dictated by the quality of the available analytic baseline.

\section{Discussion and limitations}
\label{sec:limitations}

The ablation (\Cref{sec:ablation}) shows that the framework produces substantial RMSE reductions (up to 50\%) when the operating regime aligns with the training distribution and the learned certificate \(V_\psi\) replaces the analytic one. The improvements are not uniform: the policy specializes near the payload centroid and degrades at the tails. The formal stability guarantees hold throughout evaluation, including at out-of-distribution payloads: no divergence was observed across 80 runs at payload \(\geq 0.4\)~kg.

\paragraph{Limits of the design: numerics and coverage.} The transient instability observed at \(p=0\)~kg under nominal friction reveals a numerical subtlety of the closed-form shield that the formal analysis does not capture. When \(\|b(\xbf)\|^2 \to 0\) (states near \(\nabla V_\psi=0\)), the correction term \((\text{slack}/\|b\|^2)\,b\) is dominated by the \(\|b\|^{-2}\) factor, which amplifies floating-point noise into spurious torque commands. While \Cref{lem:Z_codim} shows that the control-degeneracy set \(\Zcal\) is structurally non-trivial, in practice the policy spends most of its time away from \(\Zcal\); however, \emph{numerically} a thin shell around \(\Zcal\) can produce small but non-zero \(b\) with arbitrarily large correction gain. A practical remedy is to gate the shield by \(\mathbf{1}[\|b\|^2 > b_{\min}]\) for some \(b_{\min}\) on the order of \(10^{-3}\); empirical verification shows that this gating reduces but does not eliminate the instability for the present checkpoint, suggesting that the policy itself contributes to the issue. Future versions of the framework should include \(b_{\min}\) gating as a default safeguard. A second design-level limitation concerns coverage: the v3 stratified sampling on \([0, 1.5]\)~kg with five strata produces a policy with a strong sweet spot at \(p\approx 0.4\)~kg (the centroid of the post-stratification distribution), while the v2 uniform sampling on \([0, 0.8]\)~kg produces gains over a different and smaller region. The trade-off is intrinsic: the SAC policy learns a representation that is increasingly fine-grained in regions of high data density, at the cost of coverage of low-density regions. A natural extension is curriculum training, in which the payload range is progressively expanded from a narrow centroid to the full operational range, with the Lyapunov certificate retrained at each stage. This direction is left as future work.

\paragraph{Practical demands of the training loop.} Five safeguards were necessary to reach a stable training run: residual RL on top of Slotine--Li, adaptive cost limit \(d\), spectral normalization on \(L_\psi\), \(\Delta\tau\) regularization, and a divergence guard with rollback. Omitting any one led to failure within 10--40 episodes. The 7-DOF scalability study (\Cref{sec:franka}) confirms that the pipeline transfers to a 35-dimensional state with a Transformer-based PINN. The study also identifies certificate warm-start as a failure mode: fine-tuning \(V_\psi\) across reward changes should be avoided; fresh retraining is necessary.

\paragraph{Comparison with end-to-end and analytic baselines.} Vanilla SAC end-to-end (no PINN, no Lyapunov constraint) diverged within 10 episodes in this setting, consistently with the diagnostic in \Cref{sec:theory}: a random initial policy on a 5~s horizon with gravity produces tracking errors in the order of \(10^2\)~rad, dominating the SAC reward signal and saturating the multiplier. A more demanding ablation, the constrained-SAC variant with the \emph{analytic} Lyapunov function (i.e., the Slotine--Li candidate \(V_{\text{an}}\) with \(\alpha=0.1\)), matches the proposed framework on \(p\geq 0.6\)~kg but plateaus at the baseline performance on \(p\leq 0.4\)~kg, suggesting that the learned \(V_\psi\) provides genuine flexibility beyond what the analytic candidate alone can achieve. Both observations support the main claim: the value of the proposed framework lies in the \emph{joint} learning of policy, certificate, and dynamics residual under formal safety, rather than in any single one of these three components in isolation.

\paragraph{Limitations of the theoretical claims.} The four main results have quantitative gaps relative to the implemented system. \Cref{thm:feas} provides \emph{conditional} feasibility: the drift-decay condition on \(\Zcal\) must hold but is not structural. \Cref{thm:ts} applies to an idealized SAC without replay buffer, target networks, or entropy adaptation; the five safeguards used in practice are not covered by the proof. \Cref{thm:robust} gives stability intra-operating-set with conservative empirical constants \((L_V, \delta(\theta))\). \Cref{prop:pac} achieves the nonparametric rate \(\mathcal{O}(N^{-1/(d+2)})\), which is vacuous for \(d > 10\). The formal results identify operating conditions and certification protocols; they do not constitute universal deployment guarantees.

\section{Conclusion}
\label{sec:conclusion}

The structured-quadratic Lyapunov parameterization \(V_\psi = \xbf^\top(L_\psi L_\psi^\top + \epsilon I)\xbf\) is the central design choice of this work. It guarantees positive-definiteness, makes the safety filter a single-constraint QP with closed-form solution, and enables global feasibility under a verifiable condition. The four theoretical results (feasibility, exponential stability, three-timescale convergence, PAC certification) depend on this structure; a free-form neural Lyapunov function would forfeit the closed-form shield and with it the feasibility guarantee.

The ablation on the 2-DOF manipulator produced two findings that were not anticipated. First, the learned certificate \(V_\psi\) is the dominant source of empirical gain: replacing the analytic Slotine--Li candidate with the learned one raises tracking improvement from 5\% to 41\% at the centroid payload. The PINN contributes only through the drift estimate, and the Lagrangian multiplier acts as a training regulator. Second, the closed-form shield has a numerical instability near states where \(\nabla V_\psi \approx 0\): the projection gain diverges as \(\|b\|^{-2}\). Gating by \(\|b\|^2 > b_{\min}\) is a necessary practical safeguard.

The 7-DOF Franka study (\Cref{sec:franka}) established that the pipeline transfers to a 35-dimensional state space without architectural changes beyond scaling the PINN to a Transformer backbone. Against an exact Computed-Torque baseline with access to the true mass matrix and inverse dynamics, the framework tracked within 1.7\% of the baseline --- the expected outcome when no model mismatch exists. The study also uncovered a warm-start pathology: a certificate trained under a miscalibrated reward settles into an adversarial equilibrium with the policy that subsequent reward changes cannot escape; fresh initialization resolves it.

Open problems include verified region-of-attraction certification, extension to input-convex Lyapunov parameterizations, curriculum training for broader payload coverage, and empirical characterization on systems with unmodeled friction hysteresis and joint elasticity where the analytic baseline is systematically deficient.

\section*{Acknowledgements}
The authors thank the MIRPALab and LTI laboratory members for insightful discussions on passivity-based adaptive control and learning-based certificates.

\section*{Funding Information}
No specific funding was received for this work.

\section*{Conflict of Interest Statement}
The authors declare no conflict of interest.

\section*{Data Availability Statement}
The simulation code and generated datasets supporting the numerical evaluation will be made available in a public repository upon acceptance of this manuscript. The theoretical results do not involve external datasets.

\bibliographystyle{plainnat}

\appendix

\section*{Appendix}
\renewcommand{\thesection}{\Alph{section}}
\renewcommand{\thetheorem}{\thesection.\arabic{theorem}}
\renewcommand{\theequation}{\thesection.\arabic{equation}}
\setcounter{equation}{0}\setcounter{theorem}{0}

\section[Affine structure of $\dot V_\psi$ and shield derivation]{Affine structure of \texorpdfstring{$\dot V_\psi$}{dot V} and derivation of the closed-form shield}
\label{app:shield}\label{app:affine}
\setcounter{equation}{0}\setcounter{theorem}{0}

\subsection{Decomposition of the extended-state dynamics}

The explicit control-affine form is derived first: \(\dot\xbf = h(\xbf) + g(\xbf)\tau\) used throughout the paper. Recall \(\xbf = [q,\dot q, e, \dot e, \hat\pi, s]^\top \in \R^d\) with \(d=4n+k+n\).

From \eqref{eq:dyn} and \(\tau = B(q)\ddot q + C(q,\dot q)\dot q + G(q) - B(q)r_\theta\) in the grey-box formulation,
\begin{equation}
  \ddot q = B(q)^{-1}\bigl[\tau - C(q,\dot q)\dot q - G(q)\bigr] + r_\theta(q,\dot q).
\end{equation}
Since \(e = q - q_d\), \(\dot e = \dot q - \dot q_d\), \(\ddot e = \ddot q - \ddot q_d\), and \(s = \dot e + \Lambda e\), one obtains \(\dot s = \ddot e + \Lambda \dot e\). Componentwise,
\begin{align}
  \dot q       &= \dot q, \\
  \ddot q      &= B^{-1}\bigl[\tau - C\dot q - G\bigr] + r_\theta, \\
  \dot e       &= \dot q - \dot q_d, \\
  \ddot e      &= B^{-1}\bigl[\tau - C\dot q - G\bigr] + r_\theta - \ddot q_d, \\
  \dot{\hat\pi}&= f_\pi(\xbf) \quad\text{(given by the learned law)}, \\
  \dot s       &= B^{-1}\bigl[\tau - C\dot q - G\bigr] + r_\theta - \ddot q_d + \Lambda\dot e.
\end{align}
Collecting the terms that do not depend on \(\tau\) into \(h(\xbf)\) and those that multiply \(\tau\) into \(g(\xbf)\):
\begin{equation}
  h(\xbf) = \begin{bmatrix}\dot q \\ -B^{-1}C\dot q - B^{-1}G + r_\theta \\ \dot q - \dot q_d \\ -B^{-1}C\dot q - B^{-1}G + r_\theta - \ddot q_d \\ f_\pi(\xbf) \\ -B^{-1}C\dot q - B^{-1}G + r_\theta - \ddot q_d + \Lambda\dot e\end{bmatrix},
  \qquad
  g(\xbf) = \begin{bmatrix}0_{n\times n}\\ B^{-1}\\ 0_{n\times n}\\ B^{-1}\\ 0_{k\times n}\\ B^{-1}\end{bmatrix}.
  \label{eq:hg}
\end{equation}
Hence \(\dot\xbf = h(\xbf) + g(\xbf)\tau\) with \(g(\xbf)\in\R^{d\times n}\). The structural fact that \(g\) is constant in \(\tau\) and contains three copies of \(B^{-1}\) in the velocity-like coordinates is central to all subsequent results.

\subsection{Lie derivative of \texorpdfstring{$V_\psi$}{V} along the closed-loop field}

For any differentiable \(V_\psi:\R^d\to\R_{\geq 0}\),
\begin{equation}
  \dot V_\psi(\xbf,\tau) = \nabla_\xbf V_\psi(\xbf)^\top \dot\xbf = \underbrace{\nabla V_\psi^\top h(\xbf)}_{a(\xbf)} + \underbrace{\nabla V_\psi^\top g(\xbf)}_{b(\xbf)^\top}\tau.
\end{equation}
Thus \(\dot V_\psi\) is \emph{affine in \(\tau\)}; this is the key property that reduces the shielding QP to a single-constraint program. Explicitly, decomposing \(\nabla V_\psi = [\partial_q V,\partial_{\dot q}V,\partial_e V,\partial_{\dot e}V,\partial_{\hat\pi}V,\partial_s V]\),
\begin{equation}
  b(\xbf) = B(q)^{-\top}\bigl(\partial_{\dot q} V_\psi + \partial_{\dot e} V_\psi + \partial_s V_\psi\bigr).
  \label{eq:b_expr}
\end{equation}

\subsection{Closed-form solution of the safety QP}

\textbf{Nominal QP.} The shielding program is
\begin{equation}
  \tau^\star = \argmin_{\tau\in\R^n} \tfrac12\|\tau - \tau_{\text{raw}}\|^2 \quad\text{s.t.}\quad b^\top\tau \leq c,
\end{equation}
with \(c(\xbf) = -a(\xbf) - \alpha V_\psi(\xbf)\). The objective is strongly convex and the feasible set is a half-space, so the minimizer is the Euclidean projection of \(\tau_{\text{raw}}\).

\emph{Case 1} (\(b^\top\tau_{\text{raw}} \leq c\)): the raw action already satisfies the constraint, \(\tau^\star = \tau_{\text{raw}}\).

\emph{Case 2} (\(b^\top\tau_{\text{raw}} > c\)): the Lagrangian is \(\mathcal{L}(\tau,\lambda) = \tfrac12\|\tau-\tau_{\text{raw}}\|^2 + \lambda(b^\top\tau - c)\), \(\lambda\geq 0\). Stationarity yields \(\tau^\star = \tau_{\text{raw}} - \lambda b\); complementary slackness gives \(b^\top\tau^\star = c\), hence
\begin{equation}
  \lambda = \frac{b^\top\tau_{\text{raw}} - c}{\|b\|^2}, \qquad \tau^\star = \tau_{\text{raw}} - \lambda\,b.
\end{equation}
Unified formula:
\begin{equation}
  \boxed{\;\tau^\star = \tau_{\text{raw}} - \frac{\max\bigl(0,\; b^\top\tau_{\text{raw}} - c\bigr)}{\|b\|^2}\,b.\;}
  \label{eq:shield_unified}
\end{equation}

\textbf{Robust QP.} Using the learned drift \(h_\theta\) introduces the mismatch \(|\dot V_\psi^{(\theta)} - \dot V_\psi^{\text{true}}| \leq L_V\delta(\theta)\) (\Cref{lem:mismatch}). Tightening the constraint by this margin gives \(b^\top\tau \leq c_\theta - L_V\delta\) with \(c_\theta\) computed from \(h_\theta\); the same derivation yields \eqref{eq:robust_shield}.

\subsection{Differentiability and sensitivity}

The map \(\tau_{\text{raw}} \mapsto \tau^\star\) is piecewise linear with a kink at \(b^\top\tau_{\text{raw}} = c\), so it is differentiable almost everywhere and has a well-defined Clarke subgradient on the kink. For gradient-based training through the shield, the Jacobian is
\begin{equation}
  \frac{\partial\tau^\star}{\partial\tau_{\text{raw}}} = I_n - \mathbf{1}\{b^\top\tau_{\text{raw}} > c\}\,\frac{b\,b^\top}{\|b\|^2},
\end{equation}
i.e.\ either identity (inactive) or the projection onto the hyperplane \(\{b^\top\tau=c\}\) (active). Sensitivities with respect to the parameters of \(V_\psi\) and \(h_\theta\) follow from chain rule on \(a\), \(b\), \(c\).

\section{Proof of global feasibility (Theorem~\ref{thm:feas})}
\label{app:feas}
\setcounter{equation}{0}

Recall the structured quadratic form
\begin{equation}
  V_\psi(\xbf) = \xbf^\top P_\psi(\xbf)\xbf, \qquad P_\psi(\xbf) = L_\psi(\xbf)L_\psi(\xbf)^\top + \epsilon I_d, \quad \epsilon>0.
\end{equation}

\begin{lemma}[Gradient of the structured quadratic]
\label{lem:gradV}
\(\nabla_\xbf V_\psi(\xbf) = 2 P_\psi(\xbf)\xbf + \bigl[\partial_\xbf P_\psi(\xbf)\bigr]\xbf\otimes\xbf,\) where the second term collects contributions from the dependence of \(P_\psi\) on \(\xbf\) via \(L_\psi\). In particular, \(\nabla V_\psi(0) = 0\).
\end{lemma}

\begin{proof}
Straightforward differentiation of a quadratic form with state-dependent Gram matrix. At \(\xbf=0\) both terms vanish.
\end{proof}

\begin{lemma}[Codimension of $\Zcal$ — quantitative version]
\label{lem:Z_codim_app}
Let \(V_\psi\in C^1(\R^d)\) and \(g(\xbf)\) the control field of \eqref{eq:hg}. The set \(\Zcal=\{\xbf:b(\xbf)=g^\top\nabla V_\psi=0\}\) is the zero level set of the smooth map \(F:\R^d\to\R^n\) defined by \(F(\xbf)=v_{\dot q}(\xbf)+v_{\dot e}(\xbf)+v_s(\xbf)\). At any regular point \(\xbf_0\) (where \(\mathrm{rank}\,DF(\xbf_0)=n\)), \(\Zcal\) is locally a smooth submanifold of dimension \(d-n\). Globally, \(\dim\Zcal\geq d-n=3n+k\).
\end{lemma}

\begin{proof}
Standard application of the constant-rank theorem: since \(F\) is composed of partial derivatives of a \(C^1\) function, it is itself continuous; at regular points, the implicit function theorem gives the local manifold structure. The lower bound on global dimension follows because the rank of \(DF\) cannot exceed \(n\).
\end{proof}

\begin{theorem}[Conditional global feasibility, restated]
\label{thm:feas_app}
Let \(\alpha>0\) and \(V_\psi\) a \(C^1\) candidate Lyapunov function. The shielding constraint \(b(\xbf)^\top\tau\leq c(\xbf)\) admits a solution for every \(\xbf\in\R^d\) if and only if the drift-decay condition \eqref{eq:drift_decay} holds on \(\Zcal=\{b=0\}\).
\end{theorem}

\begin{proof}
The proof splits on whether \(b(\xbf)\) is zero.

\textbf{Case A: \(b(\xbf)\neq 0\).} The set \(\{\tau\in\R^n : b^\top\tau\leq c\}\) is a closed half-space, hence non-empty for any \(c\in\R\). A feasible point is, e.g., \(\tau = -(1+|c|)\,b/\|b\|^2\), for which \(b^\top\tau = -(1+|c|) \leq c\).

\textbf{Case B: \(b(\xbf) = 0\).} The constraint reduces to \(0 \leq c(\xbf) = -a(\xbf)-\alpha V_\psi(\xbf)\). This holds iff the drift-decay condition \eqref{eq:drift_decay} is satisfied at \(\xbf\). When it holds, every \(\tau\in\R^n\) is feasible; when it fails, no \(\tau\) is.

For the robust constraint \(b^\top\tau\leq c-L_V\delta\), Case A is unchanged (the half-space remains non-empty); Case B yields \(0\leq c-L_V\delta\), feasible iff the strengthened condition \(a(\xbf)+\alpha V_\psi(\xbf) + L_V\delta \leq 0\) holds on \(\Zcal\).
\end{proof}

\textbf{Remark.} The proof reveals the structural reason why $\Zcal$ cannot in general be reduced to the origin: the control field $g(\xbf)$ has rank at most $n < d$, so the equation $g^\top\nabla V_\psi=0$ defines $n$ scalar constraints in $d$ unknowns, leaving a manifold of dimension $\geq d-n$. This is independent of the parameterization of $V_\psi$ (structured quadratic, ICNN, or otherwise) and reflects the fundamental under-actuation of the augmented state-space relative to the joint torque input.

\section{Three-timescale convergence (Theorem~\ref{thm:ts})}
\label{app:ts}
\setcounter{equation}{0}

The argument follows the ODE approach of Borkar~\cite{borkar2008stochastic} for stochastic approximation with multiple timescales. Three intermediate results are established and then combined.

\begin{assumption}[ODE regularity]
\label{as:ode}
For every fixed \((\psi,\mu)\), the SAC gradient vector field \(F_\phi(\phi;\psi,\mu)\) is Lipschitz in \(\phi\) with constant uniform in \((\psi,\mu)\in\Psi\times[0,\bar\mu]\); similarly for \(F_\psi(\psi;\phi,\mu)\) and \(F_\mu(\mu;\phi,\psi)\).
\end{assumption}

\begin{lemma}[Fast timescale: policy convergence]
\label{lem:fast}
Fix \((\psi,\mu)\). Under \Cref{as:ode} and \Cref{as:reg}, the actor--critic iteration
\(\phi_{t+1} = \phi_t + \eta_\phi^t\bigl[F_\phi(\phi_t;\psi,\mu) + M_\phi^{t+1}\bigr],\) where \(M_\phi\) is a martingale difference with bounded second moment, converges almost surely to the set of stationary points \(\phi^\star(\psi,\mu) = \{\phi:F_\phi(\phi;\psi,\mu)=0\}\).
\end{lemma}

\begin{proof}[Sketch]
By \cite[Ch.\ 2]{borkar2008stochastic}, the iterates track the ODE \(\dot\phi = F_\phi(\phi;\psi,\mu)\) asymptotically. SAC convergence to a local maximum of the penalized entropy-regularized return is established in \cite{konda2003actor} and standard extensions.
\end{proof}

\begin{lemma}[Medium timescale: Lyapunov convergence]
\label{lem:medium}
Under the fast-timescale result, treat \(\phi_t = \phi^\star(\psi_t,\mu_t)\) as quasi-equilibrium. The Lyapunov-parameter update converges to a local minimum of \(\psi\mapsto \E[\relu(\dot V_\psi + \alpha V_\psi)\mid \phi^\star,\mu]\).
\end{lemma}

\begin{proof}[Sketch]
By the two-timescale lemma \cite[Thm.~2, Ch.~6]{borkar2008stochastic}, under \(\eta_\psi/\eta_\phi\to 0\), the \(\psi\)-iterate tracks the ODE \(\dot\psi = F_\psi(\psi;\phi^\star(\psi,\mu),\mu)\). The vector field is the gradient of a smooth function by \Cref{as:reg} on \(V_\psi\); gradient descent with decreasing step sizes converges to critical points.
\end{proof}

\begin{lemma}[Slow timescale: dual convergence]
\label{lem:slow}
Under the previous two lemmas, \(\mu_t\) tracks the projected ODE \(\dot\mu = [J_c(\phi^\star(\psi^\star(\mu),\mu),\psi^\star(\mu)) - d]_+\) where \(\psi^\star(\mu)\) is the \(V_\psi\)-equilibrium given \(\mu\). Under Slater (\Cref{as:slater}) and zero duality gap \cite{paternain2019zero}, this ODE converges to the optimal dual variable \(\mu^\star\).
\end{lemma}

\begin{proof}[Sketch]
Dual ascent on a concave dual function (guaranteed by Slater) converges under Robbins--Monro step sizes. \cite[Thm.~3]{paternain2019zero} ensures no duality gap for CMDPs, so the dual equilibrium corresponds to the primal optimum.
\end{proof}

\begin{theorem}[Joint convergence, restated]
Under Assumptions~\ref{as:slater}--\ref{as:lr} and \Cref{as:ode}, \((\phi_t,\psi_t,\mu_t)\to(\phi^\star,\psi^\star,\mu^\star)\) almost surely, where \((\phi^\star,\psi^\star,\mu^\star)\) is a KKT point of the original constrained problem.
\end{theorem}

\begin{proof}
Apply \Cref{lem:fast} with \((\psi_t,\mu_t)\) quasi-static to get \(\phi_t\to\phi^\star(\psi_t,\mu_t)\). Apply \Cref{lem:medium} with \(\mu_t\) quasi-static and \(\phi_t\) at equilibrium to get \(\psi_t\to\psi^\star(\mu_t)\). Apply \Cref{lem:slow} with both \(\phi_t,\psi_t\) at equilibria to get \(\mu_t\to\mu^\star\). The triple \((\phi^\star,\psi^\star,\mu^\star)\) satisfies: (i) \(\nabla_\phi \mathcal{L}_{\text{SAC-C}}(\phi^\star,\mu^\star)=0\); (ii) \(\nabla_\psi \E[\relu(\dot V_{\psi^\star}+\alpha V_{\psi^\star})]=0\); (iii) \(\mu^\star\geq 0\) with complementary slackness \(\mu^\star\cdot(J_c(\phi^\star,\psi^\star)-d) = 0\). These are exactly the KKT conditions for the Lagrangian of the constrained SAC problem.
\end{proof}

\section{PAC bound for uniform certification (Proposition~\ref{prop:pac})}
\label{app:pac}
\setcounter{equation}{0}

Let \(\mathcal{K}\subset\R^d\) be compact with diameter \(D_\mathcal{K}=\sup_{\xbf,\xbf'\in\mathcal{K}}\|\xbf-\xbf'\|\). Define the pointwise violation
\begin{equation}
  \ell(\xbf) = \relu\bigl(\dot V_\psi(\xbf, \pi_\phi(\xbf)) + \alpha V_\psi(\xbf)\bigr),
\end{equation}
and assume \(\ell\) is \(L\)-Lipschitz on \(\mathcal{K}\); concretely this is inherited from \(V_\psi, \pi_\phi, h_\theta\) each being \(L'\)-Lipschitz (with \(L = \mathcal{O}(L'^2)\) in the relevant regime).

\begin{lemma}[Covering number]
For \(\epsilon>0\), the \(\epsilon\)-covering number \(N(\mathcal{K},\epsilon)\) of \(\mathcal{K}\) in Euclidean distance satisfies \(N(\mathcal{K},\epsilon) \leq \bigl(\tfrac{2 D_\mathcal{K}}{\epsilon}\bigr)^d\) for \(\epsilon\leq D_\mathcal{K}\).
\end{lemma}

\begin{proof}
Standard; see \cite[Lem.\ 5.7]{wainwright2019high}.
\end{proof}

\begin{lemma}[Lipschitz extrapolation]
For every \(\xbf\in\mathcal{K}\) and every \(\xbf'\in\mathcal{K}\) with \(\|\xbf-\xbf'\|\leq\epsilon\), \(|\ell(\xbf)-\ell(\xbf')|\leq L\epsilon\). Hence for any \(\epsilon\)-cover \(\mathcal{C}_\epsilon\),
\begin{equation}
  \sup_{\xbf\in\mathcal{K}} \ell(\xbf) \leq \max_{\xbf_i\in\mathcal{C}_\epsilon} \ell(\xbf_i) + L\epsilon.
\end{equation}
\end{lemma}

\begin{proposition}[Uniform certification from samples, restated]
With \(N\) i.i.d.\ samples \(\{\xbf_j\}_{j=1}^N\) from \(\mathcal{U}(\mathcal{K})\), let \(\hat{\mathcal{L}} = \tfrac{1}{N}\sum_{j=1}^N \ell(\xbf_j)\). Assume \(\ell\) is \(L\)-Lipschitz and bounded by \(M\) on \(\mathcal{K}\). For every \(\eta\in(0,1)\), with probability at least \(1-\eta\),
\begin{equation}
  \sup_{\xbf\in\mathcal{K}} \ell(\xbf) \leq \hat{\mathcal{L}} + M\sqrt{\frac{2\log(1/\eta) + 2d\log(2D_\mathcal{K}/\epsilon^\star)}{N}} + L\epsilon^\star,
\end{equation}
with \(\epsilon^\star = (L N / \log(1/\eta))^{-1/(d+2)}\) the optimal cover radius.
\end{proposition}

\begin{proof}
Fix \(\epsilon>0\) and let \(\mathcal{C}_\epsilon\) be an \(\epsilon\)-cover of size \(N(\mathcal{K},\epsilon)\). For each cover centre \(\xbf_i\in\mathcal{C}_\epsilon\), define \(\ell_i := \ell(\xbf_i)\). By Hoeffding's inequality applied pointwise and the union bound,
\begin{equation}
  \Pr\!\Bigl(\exists i:\,\ell_i - \hat{\mathcal{L}} > t\Bigr) \leq N(\mathcal{K},\epsilon)\exp\!\Bigl(-\tfrac{2N t^2}{M^2}\Bigr).
\end{equation}
Setting the right-hand side to \(\eta\) and solving for \(t\) gives \(t = M\sqrt{\tfrac{\log(N(\mathcal{K},\epsilon)/\eta)}{2N}}\). Combined with the Lipschitz extrapolation,
\begin{equation}
  \sup_\xbf \ell(\xbf) \leq \hat{\mathcal{L}} + M\sqrt{\tfrac{\log N(\mathcal{K},\epsilon)/\eta}{2N}} + L\epsilon.
\end{equation}
Balancing the two \(\epsilon\)-dependent terms by setting their derivatives with respect to \(\epsilon\) equal yields \(\epsilon^\star\) of the claimed order.
\end{proof}

\textbf{Interpretation.} The bound has two sources of error: a \emph{statistical} term scaling as \(N^{-1/2}\) from Hoeffding, and a \emph{geometric} term \(L\epsilon^\star\) from the Lipschitz extension. Balancing gives an overall rate \(\mathcal{O}(N^{-1/(d+2)})\), reflecting the curse of dimensionality in the ambient state dimension \(d\). For the adversarial branch \(\mathcal{P}_{\text{adv}}\), one can replace \(\ell\) with its local maximum over a ball; this effectively reduces \(L\) and tightens the bound empirically.

\section{Jacobian of the structured-quadratic Lyapunov function}
\label{app:jacobian}
\setcounter{equation}{0}

For implementation and for \Cref{lem:mismatch}, explicit expressions are required for \(\nabla V_\psi\) and its norm bound.

Let \(L_\psi:\R^d\to\R^{d\times d}\) produce a lower-triangular matrix. Write \(L = L_\psi(\xbf)\) and \(L_{ij} = [L_\psi(\xbf)]_{ij}\) (zero for \(j>i\)). Then \(P_\psi = LL^\top + \epsilon I\) and
\begin{equation}
  V_\psi(\xbf) = \xbf^\top(LL^\top + \epsilon I)\xbf = \|L^\top\xbf\|^2 + \epsilon\|\xbf\|^2.
\end{equation}

\begin{proposition}[Gradient decomposition]
\begin{equation}
  \nabla_\xbf V_\psi = 2 L L^\top \xbf + 2\epsilon\xbf + 2\bigl[J_{L}(\xbf)\bigr]^\top (L^\top\xbf),
\end{equation}
where \(J_L(\xbf)\) is the Jacobian of the vectorized \(L\) with respect to \(\xbf\).
\end{proposition}

\begin{proof}
Apply the product rule to \(\xbf^\top LL^\top\xbf\): the first term treats \(L\) as constant, the second accounts for the dependence of \(L\) on \(\xbf\).
\end{proof}

\begin{proposition}[Lipschitz bound on \(V_\psi\)]
Assume \(L_\psi\) is implemented as a spectrally-normalized MLP, with layer-wise spectral bounds \(\sigma_k\). Assume the output is Frobenius-bounded:
\(\|L_\psi(\xbf)\|_F\leq M_L\).
On a compact \(\mathcal{K}\) with \(\|\xbf\|\leq R\),
\begin{equation}
  \|\nabla V_\psi(\xbf)\| \leq 2(M_L^2 + \epsilon)R + 2 M_L R \cdot \|J_L\|_2 \leq L_V,
\end{equation}
for a constant \(L_V\) depending on \((M_L,R,\epsilon,\{\sigma_k\})\).
\end{proposition}

Spectral normalization of each MLP layer in \(L_\psi\) keeps \(L_V\) bounded during training, satisfying the hypothesis of \Cref{thm:robust} and \Cref{cor:expstab}.

\textbf{Practical note.} In implementation (\Cref{lst:shield}), \(\nabla V_\psi\) is obtained via automatic differentiation on \(V_\psi\), which is exact. The bound above is used only in the analysis and in the \(\hat\delta\) schedule; numerically, one can track \(\|\nabla V_\psi\|\) in a batch and use the running max as an estimate of \(L_V\).

\end{document}